%% file: GTVMin_paper.tex
\documentclass[twoside]{article}

\usepackage[preprint]{aistats2026}

\usepackage[round,authoryear]{natbib}
\usepackage{amsmath,amssymb,amsthm,bm,graphicx,booktabs,enumitem,balance}
\usepackage{longtable}
\usepackage{booktabs}   
\usepackage{capt-of}
\usepackage{makecell}

\usepackage{graphicx}
\usepackage{float}
\usepackage{needspace}

\usepackage{xcolor} 
\setlist[itemize]{label=$\bullet$,topsep=2pt,itemsep=1pt,leftmargin=1.3em}
\usepackage[hidelinks,bookmarks=false]{hyperref}
\hypersetup{pdftitle={A Large-Dimensional Analysis of Least-Squares Federated Learning via Generalized Total Variation Minimization},pdfauthor={}}
\input{math_commands.tex}
\input{results/summary.tex}

\newtheorem{theorem}{Theorem}
\newtheorem{proposition}[theorem]{Proposition}
\newtheorem{lemma}[theorem]{Lemma}
\newtheorem{corollary}[theorem]{Corollary}
\theoremstyle{definition}
\newtheorem{assumption}{Assumption}

\theoremstyle{remark}
\newtheorem{remark}[theorem]{Remark}

\begin{document}
\runningtitle{A Large-Dimensional Analysis of Least-Squares Federated Learning}
\twocolumn[
\aistatstitle{A Large-Dimensional Analysis of Least-Squares Federated Learning via Generalized Total Variation Minimization}
\aistatsauthor{Ekkehard Schnoor \And Malik Tiomoko \And Alexander Jung}
\aistatsaddress{Aalto University \\ Espoo, Finland \And Huawei Noah's Ark Lab \\ Paris, France \And Aalto University \\ Espoo, Finland}
]

\begin{abstract}
Federated learning over decentralized networks inherently trades off the suppression of local finite-sample noise through collaboration against the bias that collaboration introduces between heterogeneous nodes. We study this tension by performing a large-dimensional analysis of Generalized Total Variation Minimization with least-squares models. Using tools from random matrix theory, we show that Lipschitz observations of the global GTVMin solution concentrate in large dimensions, universally for any underlying network structure.
We provide exact analytic expressions for the asymptotically precise performance under both i.i.d. and non-i.i.d. Gaussian mixture distributions on a fixed network.
For the latter, these expressions quantify how regularization balances local adaptation against collaboration. Our theoretical results are confirmed by experiments on synthetic and real-world data.
\end{abstract}

\section{Introduction}

Federated learning (FL) trains models across many nodes that keep their data locally and only exchange model parameters over a communication network \citep{mcmahan2017communication, kairouz2021advances}. If the nodes are statistically homogeneous (i.i.d.), the natural goal is to recover the model that would be obtained by pooling all data. In practice, however, nodes are usually heterogeneous (non-i.i.d.). A single global model may then be misspecified for individual nodes, and one has to balance local adaptation against collaboration \citep{cheng2023federated}.

Generalized Total Variation Minimization (GTVMin) formulates this trade-off on a graph \citep{jung2025federated}. At each node, a local model is trained, and a graph-Laplacian strength penalty with a regularization parameter $\alpha \geq 0$ encourages connected nodes to learn similar models. For $\alpha=0$, every local model is trained separately. As $\alpha\to\infty$, the models on a connected graph reach consensus, \ie they share the same model. Intermediate values of $\alpha$ give \textit{personalized} models. However, it is not well understood how the generalization error depends on $\alpha$, on the graph, and on the differences between the local data distributions. 
We answer this question for least-squares models in the high-dimensional regime, where the feature dimension $p$ and the local sample sizes $n_i$ grow proportionally while the number of nodes $N$ is fixed. Our analysis relies on the concentration of measures and deterministic equivalents from random matrix theory (RMT) \citep{couillet2022random}. Our main contributions are as follows.

\begin{itemize}[topsep=1pt,itemsep=0pt,leftmargin=1.3em]
  \item We show that the GTVMin solution concentrates, independently of the collaboration strength and the graph, including growing networks (Section~\ref{subsec:concentration}).
  \item We derive each node's asymptotically precise test error for both homogeneous and heterogeneous Gaussian data (Sections~\ref{subsec:homogeneous} and~\ref{subsec:heterogeneous}).
  \item We validate the predictions on synthetic data and 55 real-data settings, and use the theory to tune $\alpha$ from training data (Section~\ref{sec:experiments}).
\end{itemize}

\paragraph{Related work.}
A common approach regularizes local models toward a global anchor, as in the global/local mixture of \citet{hanzely2020federated}, pFedMe \citep{dinh2020personalized} and Ditto \citep{li2021ditto}. Optimal optimization rates for such objectives are given by \citet{hanzely2020lower}. Other approaches use meta-learning \citep{fallah2020personalized}, shared representations \citep{collins2021exploiting}, clustering \citep{sattler2020clustered, ghosh2020efficient}, mixtures of distributions \citep{marfoq2021federated}, or model interpolation \citep{mansour2020three}; see \citet{tan2022towards} for a survey.  

Coupling task-specific models through a similarity penalty goes back to regularized multi-task learning \citep{evgeniou2004regularized} and its federated version \citep{smith2017federated}. Graph-Laplacian and total-variation penalties appear in the network Lasso \citep{hallac2015network} and in decentralized collaborative learning of personalized models \citep{vanhaesebrouck2017decentralized, zantedeschi2020fully}. GTVMin \citep{jung2025federated} includes these formulations. Decentralized optimization for a single shared model is studied, e.g., by \citet{lian2017can} and \citet{koloskova2020unified}. These works focus on algorithms and convergence. We instead characterize the statistical performance along the entire regularization path.

Deterministic equivalents and concentration of measure are standard tools for the exact analysis of ridge regression and classification \citep{dobriban2018high, hastie2022surprises} and of Gaussian-mixture classifiers \citep{mai2019high, liao2019large, loureiro2021learning}. \citet{tiomoko2021deciphering} use them to analyze multi-task least-squares classifiers. Closest to our setting are analyses of distributed least squares by one-shot and iterative parameter averaging \citep{zhang2013communication, rosenblatt2016optimality, dobriban2021distributed, dobriban2020wonder}, and of federated estimators \citep{cheng2023federated}. Compared with these works, we analyze an arbitrary communication graph and a continuum of collaboration strengths. The graph enters through node-resolved resolvent blocks and, under suitable symmetry, through the Laplacian spectrum.

\section{Setup and Assumptions}

We introduce GTVMin in Section~\ref{sec:setup}, distinguish homogeneous and heterogeneous data in Section~\ref{subsec:iid_noniid},
and specify the data model and asymptotic assumptions in Section~\ref{subsec:data_model}.

\subsection{Generalized Total Variation}
\label{sec:setup}

Consider an FL network modeled as an undirected weighted graph $\mathcal{G} = (\mathcal{V}, \mathcal{E})$ with the node set $\mathcal{V} = \{1, \dots, N\}$ and the edge set $\mathcal{E}$ of unordered pairs of distinct nodes. Each node $i \in \mathcal{V}$ represents a \textit{device} (or \textit{client}) with a \emph{local dataset} $\left(\mX^{(i)},\vy^{(i)}\right)$ of size $n_i$ with a feature matrix $\mX^{(i)}\in\R^{p\times n_i}$ and a label vector $\vy^{(i)}\in\R^{n_i}$ (with shared dimension $p$), a \emph{local linear hypothesis} $h^{(i)} (\vx)=\vx^\top\vw^{(i)}$ with the local parameters $\vw^{(i)}\in\R^p$ and a \emph{local loss}  
          \begin{align}
                         L^{(i)}\left(\vw^{(i)}\right)
              = & \frac{1}{2n_i}\left\|\,\mX^{(i)\top}\vw^{(i)}-\vy^{(i)}\right\|_2^{2} \label{eq:local_loss}
          \end{align}
 
An edge $\{i,i'\}\in \mathcal{E}$, $i \neq i'$, has a weight $A_{i,i'}=A_{i',i}>0$, encoding (assumed) statistical
similarity between the nodes $i$ and $i'$, and motivating the preference $\vw^{(i)}\approx\vw^{(i')}$.
We set $A_{i,i}=0$ and $A_{i,i'}=0$ for $\{i,i'\} \notin  \mathcal{E}$. All $N$ local parameter vectors $\vw^{(i)} \in \R^p$ are ``stacked'' into one global weight vector $\vw \in \R^{Np}$,
\begin{equation}\label{eq:stack_into_w}
    \vw
=   \stack\left\{\vw^{(i)}\right\}_{i=1}^{N}
:=  \left[\vw^{(1)\top},\dots,\vw^{(N)\top}\right]^{\top}.
\end{equation}
Similarly we stack targets $\vy=\stack\{\vy^{(i)}\}_{i=1}^N\in\R^{n_\mathrm{tot}}$, and 
$\mX:=\left[\mX^{(1)},\dots,\mX^{(N)}\right]\in\R^{p\times n_{\mathrm{tot}}}$
(where $n_{\mathrm{tot}}:=\sum_{i=1}^N n_i$) is the global data matrix obtained by horizontally
concatenating all the local feature matrices.
The weighted graph Laplacian $\mL^{(\mathcal{G})}\in\R^{N\times N}$ of $\mathcal{G}$ is 
\begin{equation}\label{eq:laplacian}
    \mL^{(\mathcal{G})}_{i,i'}  = 
    \begin{cases}
        -A_{i,i'}, & \{i,i'\}\in \mathcal{E},\\[2pt]
        \sum_{\substack{i''=1\\i''\ne i}}^N A_{i,i''}, & i'=i,\\[2pt]
        0, & \text{otherwise.}
    \end{cases}
\end{equation}
Note $\mL^{(\mathcal{G})}$ is a symmetric positive-semidefinite matrix with eigenvalues
$0=\lambda_1 \leq \lambda_2 \leq \dots \leq \lambda_N$. The stacked vector
$\stack\{\vc\}_{i=1}^N=\vone_N \otimes\vc$ (any $\vc\in\R^p \setminus \{\bm{0}\} $) is an eigenvector of $\mL^{(\mathcal{G})}\otimes\mI_p$ with eigenvalue $0$,
where $\otimes$ is the Kronecker product.
The Laplacian relates to the pairwise discrepancy as follows \citep[eq. (41)]{jung2025federated}, 
\begin{align}
        \vw^{\top} \left(\mL^{(\mathcal{G})} \otimes\mI_p\right)\vw
    = & \sum_{\{i,i'\}\in \mathcal{E}}A_{i,i'}\left\|\vw^{(i)}-\vw^{(i')}\right\|_2^{2} \\
    = & \mathrm{GTV} \left( \stack \left \{ \vw^{(i)} \right \}_{i=1}^N \right).
\label{eq:gtv_quadratic_identity}
\end{align}

This is the  generalized total variation (GTV) of the local parameters across the network $\mathcal{G}$.
Combining local losses \eqref{eq:local_loss} and the penalty \eqref{eq:gtv_quadratic_identity}
leads to the objective 
\begin{equation}
   \sum_{i\in \mathcal{V}} L^{(i)}\left(\vw^{(i)}\right) + \alpha \vw^{\top}\left(\mL^{(\mathcal{G})} \otimes\mI_p\right)\vw    
\end{equation}
Ignoring terms that do not depend on $\vw$ and thus are irrelevant for optimization, define the (training) loss
\begin{equation}
    L(\vw) 
=      \tfrac{1}{2} \vw^{\top}\mB\,\vw\;-\;\vb^{\top}\vw + \alpha \vw^{\top} \left(\mL^{(\mathcal{G})} \otimes\mI_p\right)\vw,
\label{eq:objective}
\end{equation}
where $\alpha \geq 0$. The block-diagonal matrix $\mB \in \R^{Np\times Np}$ and the stacked vector $\vb \in\R^{Np}$ are given by
\begin{align}
 \mB & =\blkdiag \left(\tfrac{1}{n_i}\mX^{(i)}\mX^{(i)\top}\right)_{i=1}^{N}\in\R^{Np\times Np} ,   \\
 \vb & =\stack \left\{\tfrac{1}{n_i}\mX^{(i)}\vy^{(i)}\right\}_{i=1}^{N}\in\R^{Np}.
\end{align}
We denote the $i$th diagonal block of $\mB$ by $\mB^{(i)}$ and the $i$th block of $\vb$ by $\vb^{(i)}$, \ie
\begin{align}
\mB^{(i)} & = \tfrac{1}{n_i}\mX^{(i)}\mX^{(i)\top}\; \in \R^{p \times p} ,  \label{eq:matrix_B_i}\\
\vb^{(i)} & = \tfrac{1}{n_i}\mX^{(i)}\vy^{(i)}\; \in \R^p . \label{eq:vector_b_i}
\end{align}

Our goal is to solve the quadratic objective in \eqref{eq:objective}, \ie 
\begin{equation}
\widehat{\vw} \,  \in   \, \argmin_{\vw \in \R^{Np}} \,   L(\vw)  .
\end{equation}
It is straightforward to compute the gradient of the global loss $L$ from \eqref{eq:objective}, which is obtained as follows by
\begin{align}
    \nabla L(\vw) =  \mH \vw - \vb,
\end{align}
where the global positive-semi-definite Hessian $\mH $ is 
\begin{equation}
     \mH  =  \mB + 2 \alpha \left(\mL^{(\mathcal{G})} \otimes \mI_p \right) \; \in \R^{Np\times Np} .
   \label{eq:def_H}
\end{equation} 
Stationary points are given by zero-gradient condition
\begin{equation}\label{eq:gtvmin_zero_gradient}
    \nabla L(\widehat\vw)=\vzero
    \quad\Longleftrightarrow\quad
    \mH \widehat \vw  =  \vb,
 \end{equation}
and, if $\mH$ is invertible, we obtain the unique solution $\widehat \vw = \mH^{-1} \vb$. 
The block-diagonal matrix $\mB$ collects the empirical second-moment matrices appearing in the local linear regression, and $2\alpha(\mL^{(\mathcal{G})} \otimes\mI_p)$ \emph{couples} them across the edges of $\mathcal{G}$. 
When $\alpha=0$, then $\mH$ decouples and \eqref{eq:gtvmin_zero_gradient} reduces to $N$ independent local regressions $\tfrac{1}{n_i}\mX^{(i)}\mX^{(i)\top}\vw^{(i)}=\tfrac{1}{n_i}\mX^{(i)}\vy^{(i)}$.
As computing the inverse $\mH^{-1}$ is computationally expensive, we may employ gradient-based optimization.
For a step size $\tau>0$, define the fixed-point operator $ \Psi: \R^{Np} \to \R^{Np}$,
\begin{align}
    \label{eq:gtvmin_fixed_point}
         \Psi(\mX,\vy)(\vw)
     = & \vw-\tau\nabla L(\vw) \\
     = & \left(\mI_{Np} - \tau \mH \right) \vw + \tau \vb. \nonumber
\end{align}
By the fixed-point formulation \citep{jung2017fixed}, any solution $\widehat\vw$ of \eqref{eq:gtvmin_zero_gradient} is a fixed point, \ie $\Psi(\mX,\vy)(\widehat\vw)=\widehat\vw$. Our notation emphasizes the dependency on the random data $(\mX,\vy)$ (making $\Psi(\mX,\vy): \R^{Np} \to  \R^{Np}$ a random function), and below we investigate when the fixed-point iteration 
\begin{equation}\label{eq:gtvmin_iter}
    \vw_{t+1}
     = \Psi(\mX,\vy)(\vw_t)
     = \vw_t - \tau\left(\mH\,\vw_t-\vb\right)
\end{equation}
converges to the solution $\widehat\vw=\mH^{-1}\vb$. If $\mH\succ0$ and $0<\tau<2/\lambda_{\max}(\mH)$, its contraction factor is $\|\mI_{Np}-\tau\mH\|_{2\to2}<1$, and convergence holds from every initialization. Unrolling block by block,
\eqref{eq:gtvmin_iter} reads as
\begin{align}
    \label{eq:gtvmin_iter_block}
            \vw_{t+1}^{(i)}
     = &    \vw_t^{(i)} - \tau
        \Big[
      \tfrac{1}{n_i}\mX^{(i)} \left(\mX^{(i)\top}\vw_t^{(i)}-\vy^{(i)}\right) \\
       &  + 2\alpha \sum_{i'=1}^N A_{i,i'}\left(\vw_t^{(i)}-\vw_t^{(i')}\right)
    \Big].
\end{align}
 
Each node thus performs the same gradient step as in the local linear regression case (the first summand in the update), 
plus a weighted pull towards its graph neighbors (the second term in the update).

\subsection{The i.i.d.\ and non-i.i.d.\ regimes}
\label{subsec:iid_noniid}
The GTVMin penalty $\alpha\ge0$ in \eqref{eq:objective} controls the degree of collaboration between the nodes, and the optimal choice of $\alpha$ depends on the similarity of their data-generating processes. We distinguish between the following scenarios, the \textit{i.i.d.} and the \textit{non-i.i.d.} regimes.

\paragraph{Homogeneous data (\textit{i.i.d.}).}
All nodes share one distribution, so their population minimizers coincide: collaboration adds no population-model mismatch and suppresses finite-sample noise. Under our model with equal sample sizes and a connected graph, the asymptotic optimum is thus $\alpha\to\infty$, recovering pooled least squares over all $n_{\mathrm{tot}}$ samples.

\paragraph{Heterogeneous data (\textit{non-i.i.d.}).}
The local distributions differ, so a single consensus model need not be optimal: collaboration can suppress noise but drag each model toward an incorrect shared solution. The optimal $\alpha$ balances collaboration against personalization; Section~\ref{subsec:heterogeneous} quantifies this tradeoff without assuming an interior or unique optimum.

\paragraph{Consensus.}
On a connected graph $\mathcal{G}$ with positive-definite local second-moment matrices $\mB^{(i)}$ from \eqref{eq:matrix_B_i}, 
the consensus limit as $\alpha \to \infty$  is given by
\begin{equation}
 \left(\sum_{i=1}^N\mB^{(i)}\right)^{-1}
 \sum_{i=1}^N\vb^{(i)}.\label{eq:consensus}
\end{equation}
For unequal sample sizes this minimizes the equally weighted sum of node losses, as specified in \eqref{eq:objective}; for equal sample sizes it is pooled least squares.

\subsection{Data Model and Assumptions}
\label{subsec:data_model}

We make the following assumption on the data distribution.
At node $i$, labels $y\in\{-1,+1\}$ and features follow the Gaussian mixture model
\begin{equation}
 \vx\mid y\sim\mathcal N \left( y\vmu^{(i)},\mSigma^{(i)} \right),\qquad i=1,\ldots,N, \label{eq:model}
\end{equation}
with local means $\vmu^{(i)} \in \R^p$ and covariances $\mSigma^{(i)} \in \R^{p \times p}$.
The class sizes are arbitrary and may differ across nodes. Training data are independent within and across nodes, and the test data is independent of the training data.
Since $y^2=1$, the squared residual $(\vx^\top\vw-y)^2$ equals $(\vz^\top\vw-1)^2$ with the signed observation $\vz=y\vx$, s.t. the labels disappear from the problem. We evaluate the classifier $\vx\mapsto\sign(\vx^\top\whw^{(i)})$ on a fresh observation from node $i$. It makes an error when $\vz^\top\whw^{(i)}<0$. Given the training data, $\vz^\top\whw^{(i)}$ is Gaussian, so for a deterministic $\whw^{(i)}\ne\vzero$ the error is exactly
\begin{equation}
 \calE_i(\alpha)=\Phi\left(-\frac{\vmu^{(i)\top}\whw^{(i)}}
                 {\sqrt{\whw^{(i)\top}\mSigma^{(i)}\whw^{(i)}}}\right),\label{eq:testerror}
\end{equation}
where $\Phi$ is the standard normal distribution function. 

\begin{assumption}[Proportional dimensions]\label{ass:main}
The number of nodes $N$ and the Laplacian $\mL^{(\mathcal G)}$ are fixed, $p\to\infty$, and $c_i:=p/n_i\to\bar c_i\in(0,1)$ for $i=1,\ldots,N$. The deterministic parameters may depend on $p$, with uniformly bounded $\|\vmu^{(i)}\|_2$ and
 $s_1\mI_p\preceq\mSigma^{(i)}\preceq s_2\mI_p$ for some $0<s_1\le s_2$.  
\end{assumption}

Since $n_i>p$, each $\mB^{(i)}$ is invertible with high probability (Appendix~\ref{app:gap}), so $\mH\succ0$ and the minimizer $\whw$ is unique for every $\alpha$ and every graph. Homogeneous nodes share $\vmu^{(i)}=\vmu$, $\mSigma^{(i)}=\mSigma$ and $n_i=n$; in the heterogeneous case, all three may differ.
The error $\calE_i(\alpha)$ is random because $\whw$ depends on the training set. In high dimension it becomes essentially deterministic, and two independent training sets give almost the same test error. The signal $\vmu^{(i)\top}\whw^{(i)}$ and the noise $\whw^{(i)\top}\mSigma^{(i)}\whw^{(i)}$ in~\eqref{eq:testerror} aggregate many independent samples, each with a small influence, and their fluctuations vanish at rate $1/\sqrt{p}$ (Theorem~\ref{thm:concentration}). Sections~\ref{subsec:homogeneous} and~\ref{subsec:heterogeneous} compute their deterministic limits.

\section{Theoretical Main Results}

We first establish the concentration of the GTVMin solution in Section~\ref{subsec:concentration}. We then derive the test-error formulas for homogeneous and heterogeneous data in Sections~\ref{subsec:homogeneous} and~\ref{subsec:heterogeneous}, respectively.

\subsection{Concentration of the GTVMin Solution}
\label{subsec:concentration}

We first show that Lipschitz observables of the GTVMin solution
concentrate around their conditional means, uniformly over the
collaboration strength $\alpha$ and the graph. Write
$n_{\min}:=\min_{1\le i\le N}n_i$, and let $\mZ^{(i)}$ and
$\mS^{(i)}$ denote the signed data matrix and centered sample
covariance at node $i$. For the constants $h,K>0$ of
Appendix~\ref{app:gap}, let us define 
\begin{equation}
 \mathcal A := \bigcap_{i=1}^N \left\{ \lambda_{\min}\left(\mS^{(i)}\right) \ge h, \, \left\|\mZ^{(i)}\right\|_{2\to2} \le K\sqrt{n_i} \right\}.
 \label{eq:event_A}
\end{equation}
The bound \eqref{eq:goodprob} below implies that
$\P\left(\mathcal A\right)\to1$ as $n_{\min}\to\infty$,
provided $\log N/n_{\min}\to0$. Thus the smallest local
sample size must grow faster than the logarithm of the
number of nodes; the network itself may also grow.

\begin{theorem}[Observable concentration]
\label{thm:concentration}
Assume~\eqref{eq:model} with
$\left\|\vmu^{(i)}\right\|_2\le K_\mu$,
$s_1\mI_p\preceq\mSigma^{(i)}\preceq s_2\mI_p$ and
$p/n_i\le1-\varepsilon$ for all $i$. There are constants
$C,c>0$, depending only on $K_\mu,s_1,s_2,\varepsilon$
and not on $\alpha$, the graph or $N$, such that
\begin{equation}
 \P\left(\mathcal A^\complement\right)
 \le CN\exp\left(-c n_{\min}\right).
 \label{eq:goodprob}
\end{equation}
If $\P\left(\mathcal A\right)\ge1/2$, then for every
$\alpha\ge0$, every weighted graph and every
$f:\R^{Np}\to\R$ that is $1$-Lipschitz with respect to
$\|\,\cdot\,\|_2$,
\begin{align}
 &\P\left(
 \left|f(\whw)-\E\left[f(\whw)\mid\mathcal A\right]\right|>t
 \,\middle|\,\mathcal A
 \right)
 \nonumber\\
 &\qquad\le C\exp\left(-c n_{\min}t^2\right),
 \qquad \forall t>0.
 \label{eq:concentration}
\end{align}
\end{theorem}

In words, the theorem controls fluctuations of every $1$-Lipschitz observable around its conditional mean at scale $1/\sqrt{n_{\min}}$, 
hence $1/\sqrt p$ in the proportional regime. Since the constants do not depend on $N$, the theorem also covers growing networks. It holds for arbitrary edge weights, without bounds on node degrees or on the largest Laplacian eigenvalue.

For fixed $N$, the theorem also controls fluctuations in the test-error formula~\eqref{eq:testerror}. 
The signal $\vmu^{(i)\top}\whw^{(i)}$ is Lipschitz, and the quadratic noise $\whw^{(i)\top}\mSigma^{(i)}\whw^{(i)}$ is Lipschitz
on the bounded set of fitted coefficients, as explained in the proof idea below. 
Their Lipschitz constants do not depend on $p$, so both fluctuate at scale $1/\sqrt p$.
Together with the noise being bounded away from zero with high probability (Appendix~\ref{app:node}), 
this gives the same fluctuation scale for their ratio and the test error.
Sections~\ref{subsec:homogeneous} and~\ref{subsec:heterogeneous} derive the corresponding deterministic approximations.
Figure~\ref{fig:homogeneous} (right) illustrates the $p^{-1/2}$ scale.

\subsection{Preparations}

Let us begin by introducing some necessary notation required to compute the test error in
both the homogeneous and heterogeneous cases. With the short-hand
\begin{equation}\label{eq:matrix_M}
    \mM :=  2\alpha(\mL^{(\mathcal G)}\otimes\mI_p) \; \in \R^{Np\times Np},
\end{equation}
define $\vdelta=(\delta_1,\ldots,\delta_N)^\top\in(0,\infty)^N$
and $\mT\in\R^{Np\times Np}$ implicitly through the following fixed-point system,
\begin{align}
 \mT
 &=\left(
 \mM+\blkdiag\left(
 \frac{\mSigma^{(i)}}{1+\delta_i}
 \right)\right)^{-1},
 \label{eq:matrix_T}\\
 \delta_i
 &=\frac1{n_i}\tr\left(\mSigma^{(i)}\mT_{ii}\right),
 \qquad i=1,\ldots,N.
 \label{eq:fixedpoint}
\end{align}
Here, the symmetric positive definite matrix $\mT=(\mT_{ij})_{i,j=1}^N$ in \eqref{eq:matrix_T} is partitioned into blocks of size $p\times p$.
Furthermore, we define the matrix $\mJ\in\R^{N\times N}$ given by
\begin{equation}
 J_{ij}
 =\frac{
 \tr\left(
 \mSigma^{(i)}\mT_{ij}
 \mSigma^{(j)}\mT_{ji}
 \right)}
 {n_j(1+\delta_j)^2},
 \quad i,j=1,\ldots,N.
 \label{eq:J}
\end{equation}
This matrix is needed to express the test error. 
We need the following technical result as a preparation.
\begin{proposition}[Existence and stability]\label{prop:stability}
If $c_i<1$ for every node, \eqref{eq:fixedpoint} has a unique positive solution, with $\delta_i\le c_i/(1-c_i)$. Iteration of its right-hand side from $\vdelta=\vzero$ converges to this solution. Moreover, we obtain
\begin{align}
 \rho(\mJ)&\le\max_{1\le i\le N}c_i<1, \label{eq:spectral_radius}\\
 (\mI_N-\mJ)^{-1}&=\sum_{k=0}^{\infty}\mJ^k \geq 0.\label{eq:stability}
\end{align}
where $\rho$ denotes spectral radius in \eqref{eq:spectral_radius}, and using the Neumann series and an entrywise comparison in \eqref{eq:stability}.
\end{proposition}

\subsection{Test Error: Homogeneous Case}
\label{subsec:homogeneous}

The previous result makes the test-error formulas well defined, as the fixed point uniquely determines $\mT$ and $\mJ$, and
\eqref{eq:stability} allows solving for $\vv$ in \eqref{eq:homvar} and \eqref{eq:generalvar}. 
Entrywise nonnegativity also ensures $\vv_{\mathrm{het}} \geq 0$ in \eqref{eq:varianceparts}.
The proof in Appendix~\ref{app:stability} applies to any bounded deterministic positive-semidefinite coupling matrix $\mM$, not just a graph Laplacian. It also gives the stability needed when differentiating the fixed point.
We denote the deterministic approximation to the conditional test error $\calE_i$ by $\overline{\calE}_i$.

\begin{corollary}[Homogeneous test error]\label{cor:homogeneous}
Under Assumption~\ref{ass:main} with class-wise means $\pm \vmu \in \R^{p}$, common covariance $\mSigma \in \R^{p \times p}$, local sample size $n$ and $\gamma=\vmu^\top\mSigma^{-1}\vmu$, the homogeneous test error is given by
\begin{equation}
    \overline{\calE}_i(\alpha)
 =\Phi\left(-\frac{\gamma}{\sqrt{\gamma+(1+\gamma)^2v_i}}\right),\label{eq:homerror}   
\end{equation}
where the vector $ \vv = [v_1, \dots, v_N]^\top \in \R^N$ is given by
\begin{equation}
 \vv = \frac{1}{1+\gamma}(\mI_N-\mJ)^{-1}\mJ\vone. \label{eq:homvar}
\end{equation}

Then $\calE_i(\alpha)-\overline{\calE}_i(\alpha)\to0$ in probability and in mean. 
\end{corollary}

The parameter $\gamma$ measures class separation and is independent
of $\alpha$. In the homogeneous case, collaboration affects the
asymptotic test error only through the variance correction $v_i$
in \eqref{eq:homerror}.

\paragraph{Special Case.} The formula simplifies further when the orthogonal projector onto
each eigenspace of $\mL^{(\mathcal G)}$ has constant diagonal; this condition holds, for example, for cycles and complete graphs
with uniform edge weights.
Then all $\delta_i$ equal a common $\delta$, and every row of
$\mJ$ has the same sum $\theta$. Consequently, all nodes have the
same predicted test error. Writing
$0=\lambda_1\le\cdots\le\lambda_N$ for the Laplacian eigenvalues,
listed with multiplicity, we obtain
\begin{align}
 \mQ_k
 &=\left(
 \frac{1}{1+\delta} \mSigma
 +2\alpha\lambda_k\mI_p
 \right)^{-1}\in\R^{p\times p},
 \nonumber\\
 \delta
 &=\frac1{nN}\sum_{k=1}^N
 \tr\left(\mSigma\mQ_k\right),
 \nonumber\\
 \theta
 &=\frac1{nN(1+\delta)^2}\sum_{k=1}^N
 \tr\left(\mSigma^2\mQ_k^2\right),
 \label{eq:theta}\\
 v_i
 &=\frac{\theta}{(1+\gamma)(1-\theta)},
 \nonumber\\
 \overline{\calE}_i
 &=\Phi\left(
 -\frac{\gamma\sqrt{1-\theta}}{\sqrt{\gamma+\theta}}
 \right).
 \label{eq:spectralrisk}
\end{align}
Here $0<\theta\le c=p/n<1$. The scalar $\theta$ plays the role of an effective aspect ratio in the error formula. In this
special case, the graph enters only through its Laplacian eigenvalues. The local and pooled endpoints admit simple expressions even
without this graph symmetry. For every connected graph,
\begin{align}
 \overline{\calE}_i(0)
 &=\Phi\left(
 -\frac{\gamma\sqrt{1-c}}{\sqrt{\gamma+c}}
 \right),
 \nonumber\\
 \lim_{\alpha\to\infty}\overline{\calE}_i(\alpha)
 &=\Phi\left(
 -\frac{\gamma\sqrt{1-c/N}}{\sqrt{\gamma+c/N}}
 \right).
 \label{eq:limits}
\end{align}
These terms correspond to least-squares classification with $n$ local observations and $Nn$ pooled observations,
respectively. In the spectral formula above, these endpoints correspond to $\theta=c$ and $\theta\to c/N$.
For disconnected graphs, the pooled expression uses the size of the node's connected component in place of $N$.

\begin{figure*}[t]
 \centering
 \includegraphics[width=.48\textwidth]{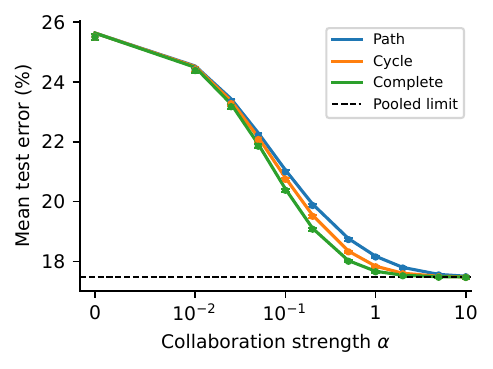}\hfill
 \includegraphics[width=.48\textwidth]{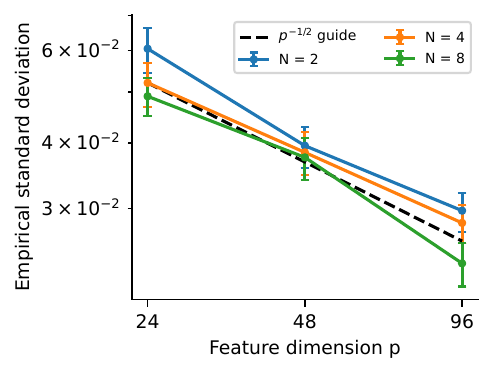}
 \caption{Homogeneous data and concentration, with common $\mSigma=\mI_p$ and $\vmu=\ve_1$. \textbf{Left:} node-average test error with $N=6$, $p=60$ and $n_i=150$. Path, cycle and complete graphs each have total edge weight 6. Lines: theoretical predictions; points: Monte Carlo means. \textbf{Right:} standard deviation of the unit-norm observable $f(\whw)=N^{-1/2}\sum_{i=1}^N\ve_1^\top\whw^{(i)}$ on an unweighted cycle, with $N\in\{2,4,8\}$, $p\in\{24,48,96\}$, $n_i=2.5p$ and $\alpha=0.3$. The $p^{-1/2}$ guide is anchored at $N=4$, $p=24$. Points use \NumGaussianRepeats\ independent training draws; error bars are 95\% intervals, with bootstrap intervals for the standard deviations.}
 \label{fig:homogeneous}
\end{figure*}

\subsection{Test Error: Heterogeneous Case}
\label{subsec:heterogeneous}

We now allow means, covariances, and sample sizes to differ
across nodes. We keep using the the fixed-point solution $\vdelta$ and
the matrix $\mJ$ from
\eqref{eq:matrix_T}--\eqref{eq:J}.
Recall that
$\mC^{(i)}=\mSigma^{(i)}+\vmu^{(i)}\vmu^{(i)\top}$
is the second-moment matrix at node $i$.
Define $\mQ\in\R^{Np\times Np}$ and $\vm\in\R^{Np}$ by
\begin{align}
 \mQ
 &=\left(
 \mM+\blkdiag\left(
 \frac{\mC^{(i)}}{1+\delta_i}
 \right)\right)^{-1},
 \label{eq:matrix_Q}\\
 \vm
 &=\mQ\stack\left\{
 \frac{\vmu^{(i)}}{1+\delta_i}
 \right\}_{i=1}^N.
 \label{eq:center}
\end{align}
The vector $\vm$ determines the approximations to linear
functions of the fitted coefficients, as stated in
Theorem~\ref{thm:risk} below. Its block
$\vm^{(i)}\in\R^p$ depends on the distributions across the
graph and can differ from the local least-squares coefficient
$\vw_*^{(i)}$.
To compute the remaining quadratic contribution to test error,
define
\begin{align}
 r_i
 &=\left(
 1-\vmu^{(i)\top}\vm^{(i)}
 \right)^2
 +\vm^{(i)\top}\mSigma^{(i)}\vm^{(i)},
 \nonumber\\
 \vv
 &=\left(\mI_N-\mJ\right)^{-1}\mJ\vr,
 \label{eq:generalvar}
\end{align}
where $\vr=\left[r_1,\ldots,r_N\right]^\top$
and $\vv=\left[v_1,\ldots,v_N\right]^\top$.
Here
\begin{equation*}
    r_i =   \E  \left[  \left(y-\vx^\top\vm^{(i)}\right)^2    \right]
\end{equation*}
is the expected squared prediction error of $\vm^{(i)}$ on a fresh observation from node $i$.
The nonnegative correction $v_i$ accounts for training-sample fluctuations in the quadratic term $\whw^{(i)\top}\mSigma^{(i)}\whw^{(i)}$.

\begin{theorem}[Heterogeneous test error]\label{thm:risk}
Under Assumption~\ref{ass:main}, for every node $i$, it holds that
\begin{align}
 \vmu^{(i)\top}\whw^{(i)}-\vmu^{(i)\top}\vm^{(i)}&\xrightarrow{\P}0,\label{eq:signalDE}\\
 \whw^{(i)\top}\mSigma^{(i)}\whw^{(i)}
       -\vm^{(i)\top}\mSigma^{(i)}\vm^{(i)}-v_i&\xrightarrow{\P}0.\label{eq:noiseDE}
\end{align}
Consequently, the heterogeneous test error is given by
\begin{equation}
 \overline{\calE}_i(\alpha)=\Phi\left(
 -\frac{\vmu^{(i)\top}\vm^{(i)}}{\sqrt{\vm^{(i)\top}\mSigma^{(i)}\vm^{(i)}+v_i}}\right),\label{eq:generalrisk}
\end{equation}
and we have $\calE_i-\overline{\calE}_i\to0$ in probability and in mean. The denominator in \eqref{eq:generalrisk} is asymptotically bounded below by a positive constant for any values of $\alpha$.
\end{theorem}

\begin{figure*}[t]
 \centering\includegraphics[width=.97\textwidth]{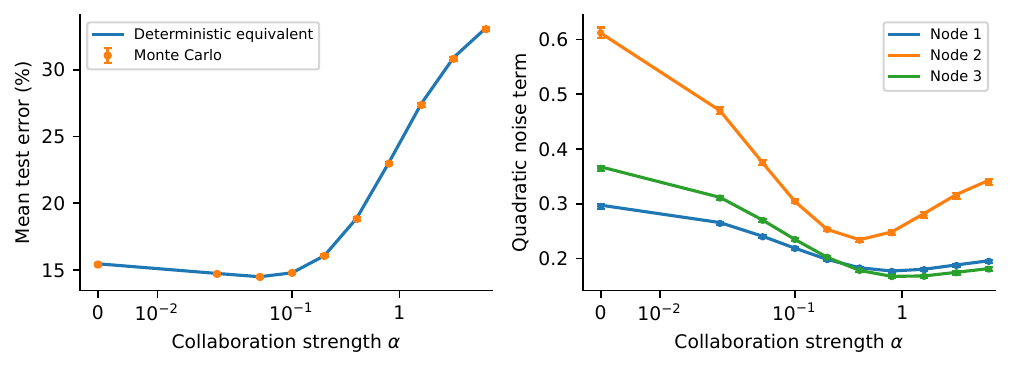}
 \caption{Heterogeneous three-node path with $p=80$, $n_i=200$, $\vmu^{(1)}=1.4\ve_1$, $\vmu^{(2)}=-0.4\ve_1+\ve_2$ and $\vmu^{(3)}=-0.7\ve_2+0.8\ve_3$. Covariances are diagonal: entries of $\mSigma^{(1)}$ and $\mSigma^{(2)}$ are evenly spaced from $0.5$ to $1.5$ and from $1.4$ to $0.6$, respectively; $\mSigma^{(3)}$ has half its diagonal entries equal to $0.4$ and half equal to $2$. \textbf{Left:} predicted and measured node-average test error. \textbf{Right:} node-specific quadratic noise terms. Lines: Theorem~\ref{thm:risk}; points and 95\% intervals: \NumGaussianRepeats\ independent training draws, shared across values of $\alpha$.}
 \label{fig:heterogeneous}
\end{figure*}

To interpret the effect of collaboration in
Theorem~\ref{thm:risk}, we compare the deterministic coefficient
$\vm^{(i)}$ with the optimal local least-squares coefficient,
\begin{align}
 \vw_*^{(i)}
 &=\argmin_{\vw\in\R^p}
 \E\left[\left(y-\vx^\top\vw\right)^2\right]
 \nonumber\\
 &=\frac{\left(\mSigma^{(i)}\right)^{-1}\vmu^{(i)}}
 {1+\gamma_i}.
 \label{eq:population}
\end{align}
The residual $r_i$ in \eqref{eq:generalvar} then separates
into the minimum local squared prediction error and the
additional error caused by $\vm^{(i)}$ differing from
$\vw_*^{(i)}$,
\begin{align}
 r_i
 &=\frac1{1+\gamma_i}+d_i,
 \nonumber\\
 d_i
 &=\left(\vm^{(i)}-\vw_*^{(i)}\right)^\top
 \mC^{(i)}
 \left(\vm^{(i)}-\vw_*^{(i)}\right)\ge0.
 \label{eq:residualrisk}
\end{align}

Writing $\vd=\left[d_1,\ldots,d_N\right]^\top$, we can decompose \eqref{eq:generalvar} as
\begin{equation}
    \vv = \vv_{\mathrm{base}} + \vv_{\mathrm{het}} \; \in \R^{N},
    \label{eq:varianceparts}
\end{equation}
where the two vectors $\vv_{\mathrm{base}}$ and $\vv_{\mathrm{het}}$ are given by   
\begin{align}
    \vv_{\mathrm{base}} = & \left(\mI_N-\mJ\right)^{-1}\mJ \stack \left \{ \frac1{1+\gamma_i} \right\}_{i=1}^N, \label{eq:v_base} \\
    \vv_{\mathrm{het}} = & \left(\mI_N-\mJ\right)^{-1}\mJ\vd. \label{eq:v_het}
\end{align}
By Proposition~\ref{prop:stability}, $\vv_{\mathrm{het}}$
is entrywise nonnegative: mismatch adds to the quadratic correction.
For $\vw_*=\stack\left\{\vw_*^{(i)}\right\}_{i=1}^N$,
\eqref{eq:center} also gives
\begin{equation}
 \vm-\vw_*
 =-2\alpha\mQ
 \left(\mL^{(\mathcal G)}\otimes\mI_p\right)\vw_*.
 \label{eq:biasidentity}
\end{equation}

For $\alpha>0$, $\vm=\vw_*$ (equivalently, $d_i=0$
for every node) holds exactly when the local least-squares
coefficients $\vw_*^{(i)}$ agree within each connected
component. The data distributions themselves may still differ.

\section{Experiments}
\label{sec:experiments}

We first confirm our theoretical predictions on Gaussian data (Section~\ref{sec:gaussian_experiments}), then assess their accuracy and use for choosing $\alpha$ on real data (Section~\ref{sec:real}).

\subsection{Gaussian Data}
\label{sec:gaussian_experiments}

Figure~\ref{fig:homogeneous} (left) shows close agreement between the predicted and measured test errors on three graph structures. Collaboration reduces the homogeneous test error towards the pooled error, with the graph determining how quickly it is approached. The right panel shows fluctuations of a Lipschitz-continuous observable, confirming the $p^{-1/2}$ scale of Theorem~\ref{thm:concentration}.

The heterogeneous example in Figure~\ref{fig:heterogeneous} has an intermediate best collaboration strength. Beyond this point, the error initially rises while the quadratic noise still decreases: the loss of local signal outweighs the noise reduction. Stronger collaboration eventually also increases the quadratic noise.

Figure~\ref{fig:alignment} isolates the importance of connecting similar tasks. The two graphs have identical Laplacian spectra but different test errors. Grouping similar tasks makes stronger collaboration useful, and the theory captures both error curves and their preferred collaboration strengths.  

\begin{figure*}[t]
 \centering
 \includegraphics[width=.9\textwidth]{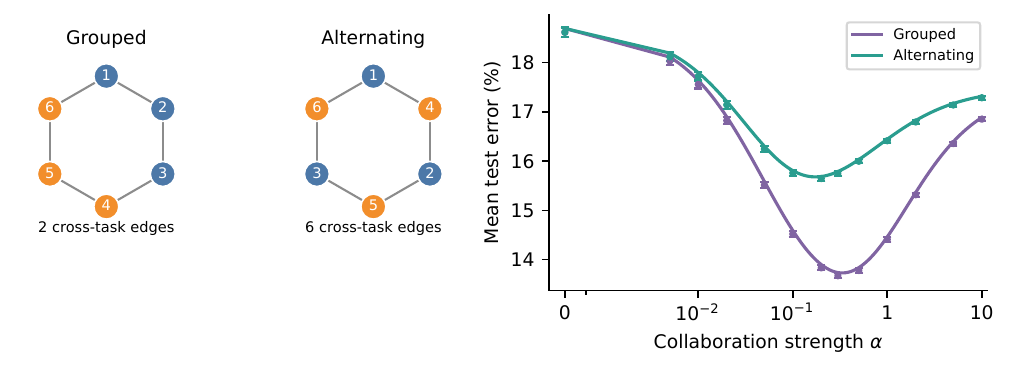}
 \caption{Six nodes with $p=60$, $n_i=150$ and $\mSigma^{(i)}=\mI_p$. Nodes $1,2,3$ have mean $\ve_1+0.8\ve_2$ (blue), and nodes $4,5,6$ have mean $\ve_1-0.8\ve_2$ (orange). Unit-weight cycles follow the orders $(1,2,3,4,5,6)$ and $(1,4,2,5,3,6)$, giving identical Laplacian spectra but two and six edges between different tasks, respectively. Lines: Theorem~\ref{thm:risk}. Points: means over \NumGaussianRepeats\ independent training draws, shared across graphs and values of $\alpha$, with 95\% intervals.}
 \label{fig:alignment}
\end{figure*}

\subsection{Real-World Data}
\label{sec:real}

We evaluate 55 settings on MNIST \citep{lecun1998gradient}, Fashion-MNIST \citep{xiao2017fashion}, CIFAR-10 \citep{krizhevsky2009learning}, HAR \citep{anguita2013public} and Adult \citep{kohavi1996scaling}. These include random splits, controlled feature and label shifts, and natural user or silo partitions. Figure~\ref{fig:calib} and Table~\ref{tab:family} show that the predictions remain useful approximations beyond the Gaussian setting.

\begin{figure}[H]
\centering
\includegraphics[width=\columnwidth]{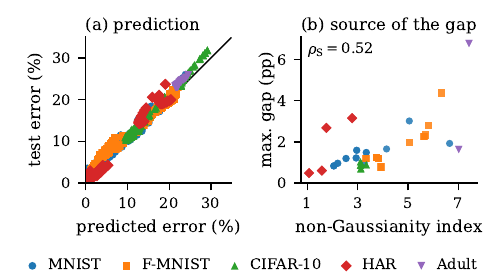}
\caption{Real-data predictions. Each node's data are split into a pool (70\%) and a held-out test set (30\%). Predictions use the signed-observation moments of the pool; measured errors average 100 independent training draws from it, shared across the $\alpha$ grid. (a) Predicted versus held-out test error for all 55 settings and grid values. (b) Largest gap to~\eqref{eq:testerror} versus the spread of whitened squared norms relative to its Gaussian value (index 1). Feature, partition and graph details are given in Appendix~\ref{app:real} and the code.}
\label{fig:calib}
\end{figure}

The prediction gaps tend to increase with departures from Gaussian data (Figure~\ref{fig:calib}b). Gaussian simulations with matching means and covariances agree more closely with the theory, suggesting that these departures explain much of the discrepancy. Held-out error can also differ from~\eqref{eq:testerror} because real-data classification scores need not be Gaussian. 

\begin{table}[h]
\centering\small
\caption{Prediction gaps and tuning regrets on 55 real-data settings, in percentage points (pp). Gap: median across settings of the largest absolute prediction error over the $\alpha$ grid. ``Gauss'' evaluates~\eqref{eq:testerror} on the fitted coefficients with pool moments; ``test'' uses held-out classification error. Regret: extra held-out error at the theory-selected $\alpha$ relative to the lowest mean error on the grid, summarized by its median and maximum. Predictions use pool moments.}
\label{tab:family}
\vspace{2pt}
\setlength{\tabcolsep}{4.5pt}
\begin{tabular}{@{}lrrrrr@{}}
\toprule
 & & \multicolumn{2}{c}{Gap (pp)}
   & \multicolumn{2}{c}{Regret (pp)} \\
\cmidrule(lr){3-4}\cmidrule(l){5-6}
Group & \# & Gauss & test & med. & max \\
\midrule
i.i.d.\ splits       & 9  & 0.8 & 1.8 & 0.07 & 0.17 \\
Dirichlet, $a=100$   & 6  & 1.0 & 2.0 & 0.12 & 0.18 \\
Dirichlet, $a\le0.5$ & 12 & 1.1 & 2.5 & 0.06 & 0.25 \\
Rotations           & 12 & 2.1 & 2.3 & 0.05 & 0.24 \\
Task alignment      & 4  & 1.1 & 2.2 & 0.02 & 0.18 \\
HAR users           & 8  & 1.6 & 3.2 & 0.14 & 1.29 \\
Adult silos         & 4  & 4.2 & 1.6 & 0.15 & 0.19 \\
\midrule
All                 & 55 & 1.2 & 2.3 & 0.09 & 1.29 \\
\bottomrule
\end{tabular}
\end{table}

\newpage

Choosing $\alpha$ by minimizing the predicted error usually gives a test error close to the best grid value (Table~\ref{tab:family}). With moments estimated only from the training data, the resulting rule also performs similarly to cross-validation (Table~\ref{tab:tuning}). Comparisons with local training, consensus, fine-tuning and global/local mixtures show that the useful strength and form of collaboration depend on the differences between nodes: graphs following the task structure perform well on the heterogeneous benchmarks, whereas consensus is best on nearly homogeneous Adult.

\begin{table}[h]
\centering
\caption{Test error (\%) on six real-data settings, averaged over 20 training draws.  Plug-in theory uses training-sample means and Ledoit--Wolf covariance estimates; oracle theory uses pool moments. Fine-tuning starts from consensus; mixture is the global/local method of \citet{hanzely2020federated}. The 95\% intervals are at most $\pm0.4$, except $\pm0.7$ for local HAR. Bold: lowest reported error excluding hindsight, which gives the lowest measured GTVMin error on the tested grid. Fore the setup details, see Appendix~\ref{app:real}.}
\label{tab:tuning}
\vspace{2pt}
\setlength{\tabcolsep}{3pt}
\resizebox{\columnwidth}{!}{%
\begin{tabular}{@{}lcccccc@{}}
\toprule
Method
 & \makecell{MNIST\\rot.}
 & \makecell{FMN.\\rot.}
 & \makecell{CIFAR\\subp.}
 & \makecell{CIFAR\\Dir.}
 & \makecell{HAR\\users}
 & \makecell{Adult\\silos} \\
\midrule
Local
 & 10.97 & 22.32 & 17.35 & 15.41 & 20.77 & 25.77 \\
Consensus
 & 18.43 & 21.88 & 12.19 & 11.43 & 14.85 & \textbf{23.15} \\
Fine-tuning
 & 9.44 & 17.58 & 11.88 & 10.65 & 14.49 & 23.44 \\
Mixture
 & 9.08 & \textbf{17.35} & 11.84 & 10.58 & 14.74 & 23.56 \\
\midrule
\multicolumn{7}{@{}l}{\emph{GTVMin, $\alpha$ chosen by}} \\
\quad CV (5-fold)
 & 8.47 & 17.90 & 11.64 & 10.52 & 14.37 & 23.51 \\
\quad Plug-in theory
 & \textbf{8.37} & 17.86 & 11.81 & 10.44 & 14.39 & 23.78 \\
\quad Oracle theory
 & \textbf{8.37} & 17.83 & \textbf{11.53}
 & \textbf{10.37} & \textbf{14.22} & 23.29 \\
\midrule
\quad Best (hindsight)
 & \textit{8.29} & \textit{17.73} & \textit{11.45}
 & \textit{10.25} & \textit{13.91} & \textit{22.95} \\
\bottomrule
\end{tabular}}
\end{table}

\clearpage
\begingroup
\raggedbottom
\setlength{\bibsep}{2pt}
\bibliographystyle{plainnat}
\bibliography{references}
\endgroup
\onecolumn
\appendix

\section{Gaussian Decomposition and Spectral Bounds}
\label{app:preliminaries}
\label{app:signed}
\label{app:gap}

The test-error proof uses independent Gaussian sample means and centered covariances, together with a bound on the inverses. Unless stated otherwise, Assumption~\ref{ass:main} holds, $N$ is fixed, and limits are taken as $p,n_i\to\infty$. Deterministic parameters, including $\alpha$, may depend on $p$. 
We use $\|\,\cdot\,\|_{2\to2}$ for the operator norm, and $\|\,\cdot\,\|_F$ denotes the Frobenius norm.

Write $\vz_r^{(i)}=y_r^{(i)}\vx_r^{(i)}$ and $\mZ^{(i)}=[\vz_1^{(i)},\ldots,\vz_{n_i}^{(i)}]\in\R^{p\times n_i}$. The sample mean and centered covariance are
\begin{equation}
 \vb^{(i)}=\frac1{n_i}\mZ^{(i)}\vone,
 \qquad
 \mS^{(i)}=\frac1{n_i}\mZ^{(i)}\mZ^{(i)\top}
 -\vb^{(i)}\vb^{(i)\top}.
 \label{eq:scatter}
\end{equation}

\begin{lemma}[Gaussian decomposition]
\label{lem:decomposition}
The sample-mean fluctuations $\veta^{(i)}:=\vb^{(i)}-\vmu^{(i)}$ are independent of all centered covariances, and different nodes are independent. Moreover,
\begin{equation}
 \veta^{(i)}\sim\mathcal N\left(\vzero,\frac{\mSigma^{(i)}}{n_i}\right),
 \qquad
 \mS^{(i)}\overset{d}{=}\frac1{n_i}
 \left(\mSigma^{(i)}\right)^{1/2}\mG^{(i)}\mG^{(i)\top}
 \left(\mSigma^{(i)}\right)^{1/2},
 \qquad
 \mB^{(i)}=\mS^{(i)}+\vb^{(i)}\vb^{(i)\top},
 \label{eq:samplemean}
\end{equation}
where $\mG^{(i)}\in\R^{p\times(n_i-1)}$ has independent standard Gaussian entries. Here $\overset{d}{=}$ means equality in distribution.
\end{lemma}

\begin{proof}[Proof of Lemma~\ref{lem:decomposition}]
Under \eqref{eq:model}, the signed observations are independent $\mathcal N\left(\vmu^{(i)},\mSigma^{(i)}\right)$ vectors, irrespective of the class proportions. An orthogonal transformation of the sample indices, with first coordinate $n_i^{-1/2}\vone$, separates the sample mean from $n_i-1$ independent centered Gaussian columns. This proves the distributional and independence statements; the identity for $\mB^{(i)}$ follows from \eqref{eq:scatter}.
\end{proof}

\begin{remark} 
\label{rem:covariance-bounds}
Standard Gaussian singular-value bounds \citep{Vershynin_2026} and the aspect-ratio gap give constants $h,K>0$ such that, with probability tending to one,
\begin{equation}
 \lambda_{\min}\left(\mS^{(i)}\right)\ge h,
 \qquad
 \left\|\mZ^{(i)}\right\|_{2\to2}\le K\sqrt{n_i},
 \qquad i=1,\ldots,N.
 \label{eq:localgap}
\end{equation}
The constants depend only on the distributional bounds and the aspect-ratio gap. On the event (in the limit $p \to \infty$)
\begin{equation}
 \mathcal B=\left\{\min_{1\le i\le N}
 \lambda_{\min}\left(\mS^{(i)}\right)\ge h\right\},
 \qquad \P\left(\mathcal B\right)\longrightarrow1,
 \label{eq:scatterevent}
\end{equation}
both $\left(\blkdiag\left(\mS^{(i)}\right)+\mM\right)^{-1}$ and $\mH^{-1}$, recalling \eqref{eq:def_H}, have operator norm at most $h^{-1}$ for every $\mM\succeq0$. This suffices for the test-error proof. The upper data bound and the exponential probability bound behind \eqref{eq:localgap} are used only for observable concentration.
\end{remark}

\section{The Fixed Point and Deterministic Equivalents}
\label{app:stability}

The fixed point approximates an inverse; its derivative supplies the correction needed for quadratic forms. We prove stability before giving these approximations. Here $\mM\in\R^{Np\times Np}$ may be any deterministic positive-semidefinite matrix; the test-error proofs use \eqref{eq:matrix_M}.

\subsection{Existence and stability}

\begin{proof}[Proof of Proposition~\ref{prop:stability}]
For $\vx\in[0,\infty)^N$, let $F_i\left(\vx\right)$ denote the right-hand side of \eqref{eq:fixedpoint}, with $\delta_j$ replaced by $x_j$. Matrix inverse monotonicity shows that $F$ is coordinatewise nondecreasing and
\begin{equation}
 0<F_i\left(\vx\right)\le c_i\left(1+x_i\right).
 \label{eq:Fbound}
\end{equation}
Iteration therefore increases within the interval $0\le x_i\le c_i/(1-c_i)$, and continuity gives a positive fixed point. For $t>1$,
\begin{equation*}
 \mM+\blkdiag\left(\frac{\mSigma^{(i)}}{1+tx_i}\right)
 \succ\frac1t\left(\mM+\blkdiag\left(\frac{\mSigma^{(i)}}{1+x_i}\right)\right),
\end{equation*}
so $F\left(t\vx\right)<tF\left(\vx\right)$ coordinatewise. If two positive fixed points $\vx,\vy$ differed, interchange them if necessary so that $t=\max_{1\le i\le N}x_i/y_i>1$. At a maximizing coordinate, monotonicity would give $x_i\le F_i\left(t\vy\right)<ty_i=x_i$, a contradiction.
To relate $\mJ$ to this fixed point, define $\mK,\mJ_\delta\in\R^{N\times N}$ by
\begin{equation}
 K_{ij}=\tr\left(\mSigma^{(i)}\mT_{ij}\mSigma^{(j)}\mT_{ji}\right),
 \qquad
 \left(\mJ_\delta\right)_{ij}
 =\frac{K_{ij}}{n_i\left(1+\delta_j\right)^2}.
 \label{eq:KJac}
\end{equation}
The matrix $\mJ_\delta$ is the Jacobian of $F$ at $\vdelta$. Also, $\mK$ is symmetric and entrywise nonnegative because
\begin{equation}
  K_{ij}=\left\|\left(\mSigma^{(i)}\right)^{1/2}\mT_{ij}\left(\mSigma^{(j)}\right)^{1/2}\right\|_F^2.
\end{equation}

For $\vomega=\left(1+\delta_1,\ldots,1+\delta_N\right)^\top$, the identity
\begin{equation*}
 \mT\blkdiag\left(\frac{\mSigma^{(i)}}{1+\delta_i}\right)\mT
 =\mT-\mT\mM\mT\preceq\mT
\end{equation*}
gives
\begin{equation}
 \left(\mJ_\delta\vomega\right)_i\le\delta_i
 \le c_i\omega_i.
 \label{eq:weightedcontraction}
\end{equation}
Thus $\rho\left(\mJ_\delta\right)\le c_{\max}:=\max_{1\le i\le N}c_i<1$. Since
\begin{equation}
 \mJ_\delta=\mD_n^{-1}\mJ\mD_n,
 \qquad \mD_n=\diag\left(n_1,\ldots,n_N\right),
 \label{eq:similarity}
\end{equation}
the same bound holds for $\mJ$, proving \eqref{eq:spectral_radius}. Its entries are nonnegative, so the convergent Neumann series gives \eqref{eq:stability}.
Under Assumption~\ref{ass:main}, the weights $\omega_i$ and ratios $n_i/n_j$ are bounded, and $c_{\max}$ stays below one. Thus \eqref{eq:weightedcontraction} uniformly bounds both $\left(\mI_N-\mJ_\delta\right)^{-1}$ and $\left(\mI_N-\mJ\right)^{-1}$. Together with $\left\|\mT\right\|_{2\to2}\le\left[s_1\left(1-c_{\max}\right)\right]^{-1}$, implicit differentiation bounds the first two derivatives under bounded smooth variations of the covariance contributions or of $\mM$.
\end{proof}

\subsection{Deterministic equivalents}
\label{app:bulk}
\label{app:response}

The following result is a block-wise version of \citep[Theorem~3.14]{louart2018concentration}.

\begin{lemma}[Deterministic equivalent of the inverse]
\label{lem:bulk}
The matrix $\mT$ in \eqref{eq:matrix_T}--\eqref{eq:fixedpoint} is a deterministic equivalent of
\begin{equation*}
 \mR=\left(\blkdiag\left(\mS^{(i)}\right)+\mM\right)^{-1}
 \in\R^{Np\times Np}
\end{equation*}
for bilinear forms with bounded deterministic vectors and traces normalized by $p$ with bounded deterministic test matrices: the corresponding scalar differences tend to zero in probability.
\end{lemma}

Similarly, the following result is obtained. Let $\mE_i=\ve_i\otimes\mI_p\in\R^{Np\times p}$ insert a vector into block $i$, where $\ve_i\in\R^N$ is the $i$th coordinate vector. Set $\mP_\ell=\mE_\ell\mSigma^{(\ell)}\mE_\ell^\top\in\R^{Np\times Np}$, so that $\vw^\top\mP_\ell\vw=\vw^{(\ell)\top}\mSigma^{(\ell)}\vw^{(\ell)}$.

\begin{lemma}[Deterministic equivalent for the quadratic term]
\label{lem:quadratic}
For $\ell=1,\ldots,N$, define
\begin{align}
 \vu^{(\ell)}
 &=\left(\mI_N-\mJ_\delta\right)^{-1}
 \stack\left\{\frac{K_{i\ell}}{n_i}\right\}_{i=1}^N
 \in\R^N,
 \label{eq:quadratic-u}\\
 \mD_\ell
 &=\mT\left(\mP_\ell+\blkdiag\left(
 \frac{u_j^{(\ell)}\mSigma^{(j)}}{\left(1+\delta_j\right)^2}
 \right)\right)\mT\in\R^{Np\times Np}.
 \label{eq:quadratic-D}
\end{align}
Then $\mD_\ell$ is a deterministic equivalent of $\mR\mP_\ell\mR$. 
Moreover, $\tr\left(\mSigma^{(j)}\mD_{\ell,jj}\right)/n_j=u_j^{(\ell)}$, where $\mD_{\ell,jj}$ denotes its $j$th diagonal block.
\end{lemma}

\section{Test Error: Heterogeneous and Homogeneous Cases}
\label{app:testerror}

We first derive the general test-error formula and the identities describing mismatch between nodes. The homogeneous corollary then follows directly. 

\subsection{Heterogeneous case}
\label{app:risk}
\label{app:heterogeneous}
\label{app:node}

\begin{proof}[Proof of Theorem~\ref{thm:risk}]
Using the Gaussian decomposition in \eqref{eq:samplemean}, define
\begin{equation}
 \mU=\left[\mE_1\vmu^{(1)},\ldots,\mE_N\vmu^{(N)}\right],
 \qquad
 \mV=\left[\mE_1\veta^{(1)},\ldots,\mE_N\veta^{(N)}\right],
 \qquad \mF=\mU+\mV\in\R^{Np\times N}.
 \label{eq:UV}
\end{equation}
Then $\mB=\blkdiag\left(\mS^{(i)}\right)+\mF\mF^\top$ and $\vb=\mF\vone$. The Woodbury identity gives the exact representation
\begin{equation}
 \whw=\mR\mF\left(\mI_N+\mF^\top\mR\mF\right)^{-1}\vone.
 \label{eq:woodbury}
\end{equation}
Condition on the centered covariances. The sample-mean fluctuations remain independent Gaussian vectors. Gaussian second moments show that $\mU^\top\mR\mV$ vanishes and $\mV^\top\mR\mV$ differs negligibly from $\diag\left(\tr\left(\mSigma^{(i)}\mR_{ii}\right)/n_i\right)$: each centered entry has conditional variance at most $C/p$ on $\mathcal B$. Since $N$ is fixed, Lemma~\ref{lem:bulk} gives convergence in operator norm,
\begin{equation}
 \mF^\top\mR\mF-\mGamma-\diag\left(\vdelta\right)
 \xrightarrow{\P}0,
 \qquad \mGamma=\mU^\top\mT\mU\in\R^{N\times N}.
 \label{eq:smallGram}
\end{equation}
Consequently, for the deterministic vector
\begin{equation}
 \vq=\left(\diag\left(1+\delta_1,\ldots,1+\delta_N\right)
 +\mGamma\right)^{-1}\vone\in\R^N,
 \label{eq:q}
\end{equation}
we have
\begin{equation}
 \left\|\whw-\mR\mF\vq\right\|_2\xrightarrow{\P}0.
 \label{eq:freeze-q}
\end{equation}
Indeed, both matrices inverted in \eqref{eq:woodbury} and \eqref{eq:q} are bounded below by $\mI_N$, and $\left\|\mF\right\|_{2\to2}$ is bounded in probability because $\E\left\|\veta^{(i)}\right\|_2^2=\tr\left(\mSigma^{(i)}\right)/n_i$ is bounded. The random part $\mV$ must still be retained: it vanishes in bounded linear forms but contributes to quadratic forms.
For any bounded deterministic $\va\in\R^{Np}$, the conditional variance of $\va^\top\mR\mV\vq$ is at most $C/n_{\min}$ on $\mathcal B$. Lemma~\ref{lem:bulk} and \eqref{eq:freeze-q} therefore give $\va^\top\whw-\va^\top\mT\mU\vq\xrightarrow{\P}0$. Set $\vm=\mT\mU\vq$. Equation~\eqref{eq:q} implies
\begin{equation}
 q_i=\frac{1-\vmu^{(i)\top}\vm^{(i)}}{1+\delta_i}.
 \label{eq:qidentity}
\end{equation}
Substituting this into $\mT^{-1}\vm=\mU\vq$ identifies $\vm$ with \eqref{eq:center}. Taking $\va=\mE_i\vmu^{(i)}$ proves \eqref{eq:signalDE}.
For the quadratic term, use the same conditional Gaussian calculation with $\mA=\mR\mP_\ell\mR$. Equation~\eqref{eq:freeze-q} also holds in bounded quadratic forms, since the vectors involved are bounded in probability. Lemma~\ref{lem:quadratic} gives
\begin{equation*}
 \whw^\top\mP_\ell\whw-
 \left(\vq^\top\mU^\top\mD_\ell\mU\vq
 +\sum_{j=1}^N q_j^2u_j^{(\ell)}\right)\xrightarrow{\P}0.
\end{equation*}
By \eqref{eq:quadratic-D}, \eqref{eq:qidentity}, and $\vm=\mT\mU\vq$, the expression in parentheses is
\begin{equation}
 \vm^{(\ell)\top}\mSigma^{(\ell)}\vm^{(\ell)}
 +\sum_{j=1}^N\frac{u_j^{(\ell)}}{\left(1+\delta_j\right)^2}
 \left[\left(1-\vmu^{(j)\top}\vm^{(j)}\right)^2
 +\vm^{(j)\top}\mSigma^{(j)}\vm^{(j)}\right].
 \label{eq:quadcalculation}
\end{equation}
The bracket is precisely $r_j$ in \eqref{eq:generalvar}. To identify its coefficient, let $\mW\in\R^{N\times N}$ have entries $W_{\ell j}=u_j^{(\ell)}/\left(1+\delta_j\right)^2$. Equation~\eqref{eq:quadratic-u} and symmetry of $\mK$ give $\mW=\mJ+\mW\mJ$, and hence
\begin{equation}
 \mW=\mJ\left(\mI_N-\mJ\right)^{-1}
 =\left(\mI_N-\mJ\right)^{-1}\mJ,
 \qquad \mW\vr=\vv.
 \label{eq:varianceidentity}
\end{equation}
This proves \eqref{eq:noiseDE}, also for unequal sample sizes.
Finally, the error denominator stays positive. Write $\gamma_i=\vmu^{(i)\top}\left(\mSigma^{(i)}\right)^{-1}\vmu^{(i)}$ and let $K_\mu$ bound the mean norms. Completing the square using \eqref{eq:population} gives
\begin{equation}
 r_i=\frac1{1+\gamma_i}
 +\left(\vm^{(i)}-\vw_*^{(i)}\right)^\top
 \mC^{(i)}\left(\vm^{(i)}-\vw_*^{(i)}\right)
 \ge\frac1{1+K_\mu^2/s_1}.
 \label{eq:rpositive}
\end{equation}
Restricting the variational formula for $\vx^\top\mT_{ii}\vx$ to vectors constant across nodes, which are annihilated by $\mM$, gives
\begin{equation*}
 \mT_{ii}\succeq
 \left(\sum_{j=1}^N\frac{\mSigma^{(j)}}{1+\delta_j}\right)^{-1}
 \succeq\frac1{Ns_2}\mI_p.
\end{equation*}
It follows that
\begin{equation}
 J_{ii}\ge\frac{c_i\left(1-c_i\right)^2s_1^2}{N^2s_2^2},
 \qquad v_i\ge J_{ii}r_i>0.
 \label{eq:positivevariance}
\end{equation}
Under Assumption~\ref{ass:main}, these lower bounds stay positive, independently of $\alpha$. The deterministic signal and quadratic terms are bounded, so continuity of $\left(a,b\right)\mapsto\Phi\left(-a/\sqrt b\right)$ for $b$ bounded away from zero turns \eqref{eq:signalDE}--\eqref{eq:noiseDE} into \eqref{eq:generalrisk}. The errors are bounded, so convergence in probability also gives convergence of their absolute difference in mean.
\end{proof}

\begin{proof}[Proof of \eqref{eq:residualrisk}--\eqref{eq:biasidentity}]
The residual identity is \eqref{eq:rpositive}. Applying $\left(\mI_N-\mJ\right)^{-1}\mJ$ gives the variance decomposition \eqref{eq:varianceparts}--\eqref{eq:v_het}, with both terms in $\R^N$; entrywise nonnegativity gives $\vv_{\mathrm{het}}\ge0$. To prove \eqref{eq:biasidentity}, subtract $\mC^{(i)}\vw_*^{(i)}=\vmu^{(i)}$ from \eqref{eq:center} to obtain
\begin{equation*}
 \left(\mM+\blkdiag\left(\frac{\mC^{(i)}}{1+\delta_i}\right)\right)
 \left(\vm-\vw_*\right)=-\mM\vw_*.
\end{equation*}
Left-multiplying the preceding equation by $\mQ$, the inverse
matrix defined in \eqref{eq:matrix_Q}, gives
$\vm-\vw_*=-\mQ\mM\vw_*$. Substituting
$\mM=2\alpha\left(\mL^{(\mathcal G)}\otimes\mI_p\right)$
yields \eqref{eq:biasidentity}.
For $\alpha>0$, $\vm=\vw_*$ is therefore equivalent to $\left(\mL^{(\mathcal G)}\otimes\mI_p\right)\vw_*=\vzero$. 
Now, \eqref{eq:gtv_quadratic_identity} with $\vw = \vw_*$
shows that this occurs exactly when the local coefficients agree within every connected component. Otherwise some $d_i>0$, and $\left(\vv_{\mathrm{het}}\right)_i\ge J_{ii}d_i>0$ by \eqref{eq:positivevariance}.
\end{proof}

\subsection{Homogeneous case}
\label{app:homogeneous}

\begin{proof}[Proof of Corollary~\ref{cor:homogeneous}]
Suppose $\vmu^{(i)}=\vmu$, $\mSigma^{(i)}=\mSigma$, and $n_i=n$. The vector $\vone\otimes\vw_*\in\R^{Np}$, with $\vw_*=\mSigma^{-1}\vmu/(1+\gamma)$, solves \eqref{eq:center}, since the Laplacian annihilates constant node vectors and $\left(\mSigma+\vmu\vmu^\top\right)\vw_*=\vmu$. Uniqueness gives $\vm^{(i)}=\vw_*^{(i)}$ and
\begin{equation*}
 \vmu^\top\vm^{(i)}=\frac{\gamma}{1+\gamma},
 \qquad
 \vm^{(i)\top}\mSigma\vm^{(i)}=\frac{\gamma}{\left(1+\gamma\right)^2},
 \qquad r_i=\frac1{1+\gamma}.
\end{equation*}
Substitution into \eqref{eq:generalvar} and \eqref{eq:generalrisk} proves \eqref{eq:homvar} and \eqref{eq:homerror}, with the convergence asserted in Theorem~\ref{thm:risk}.
\end{proof}

\section{Observable Concentration}
\label{app:concentration}

This separate result additionally uses the upper data bound in \eqref{eq:localgap} and allows growing $N$ under the uniform assumptions of Theorem~\ref{thm:concentration}. It is not needed for the test-error derivation.

\begin{proof}[Proof of Theorem~\ref{thm:concentration}]
The Gaussian singular-value bounds used in Remark~\ref{rem:covariance-bounds} give
\begin{equation*}
 \P\left\{\lambda_{\min}\left(\mS^{(i)}\right)<h
 \ \text{or}\ \left\|\mZ^{(i)}\right\|_{2\to2}>K\sqrt{n_i}\right\}
 \le C e^{-a n_i}.
\end{equation*}
A union bound proves \eqref{eq:goodprob} for the event $\mathcal A$ defined in the main text. On this event, choose a node $i$ attaining $r=\max_{1\le j\le N}\left\|\whw^{(j)}\right\|_2>0$ and set $\vu=\whw^{(i)}/r$. Taking the inner product of its normal equation with $\vu$ gives
\begin{equation*}
 r\vu^\top\mB^{(i)}\vu
 +2\alpha\sum_{j=1}^N A_{ij}\left(r-\vu^\top\whw^{(j)}\right)
 =\vu^\top\vb^{(i)}\le K.
\end{equation*}
Every graph term is nonnegative. Since $\mB^{(i)}\succeq h\mI_p$, this proves
\begin{equation}
 \max_{1\le i\le N}\left\|\whw^{(i)}\right\|_2\le K/h,
 \label{eq:blockbound}
\end{equation}
also when $r=0$, with no dependence on the graph or $N$.
Equip the full signed dataset $\mZ=\left(\mZ^{(i)}\right)_{i=1}^N$ with the product Frobenius norm, whose square is $\sum_{i=1}^N\left\|\mZ^{(i)}\right\|_F^2$. Consider two realizations $\mZ_1=\left(\mZ_1^{(i)}\right)_{i=1}^N$ and $\mZ_2=\left(\mZ_2^{(i)}\right)_{i=1}^N$ in $\mathcal A$, and denote their fitted coefficients and sample quantities by corresponding subscripts. Subtracting their normal equations gives
\begin{equation*}
 \mH_1\left(\whw_1-\whw_2\right)
 =\vb_1-\vb_2-\left(\mB_1-\mB_2\right)\whw_2.
\end{equation*}
The data bound and \eqref{eq:blockbound} imply
\begin{align*}
 \left\|\vb_1^{(i)}-\vb_2^{(i)}\right\|_2
 &\le\frac{\left\|\mZ_1^{(i)}-\mZ_2^{(i)}\right\|_F}{\sqrt{n_i}},\\
 \left\|\left(\mB_1^{(i)}-\mB_2^{(i)}\right)\whw_2^{(i)}\right\|_2
 &\le\frac{2K^2}{h\sqrt{n_i}}
 \left\|\mZ_1^{(i)}-\mZ_2^{(i)}\right\|_F.
\end{align*}
Using $\left\|\mH_1^{-1}\right\|_{2\to2}\le h^{-1}$ and summing the squared block bounds yields
\begin{equation}
 \left\|\whw_1-\whw_2\right\|_2
 \le\ell\left\|\mZ_1-\mZ_2\right\|_F,
 \qquad \ell=\frac{1+2K^2/h}{h\sqrt{n_{\min}}}.
 \label{eq:solutionlip}
\end{equation}
Thus, for a $1$-Lipschitz $f$, the function $f\left(\whw\left(\mZ\right)\right)$ is $\ell$-Lipschitz on $\mathcal A$. Extend it to a real-valued function $\widetilde f$ on the full data space with the same Lipschitz constant; such an extension always exists. The signed data are jointly Gaussian with covariance operator norm at most $s_2$, so Gaussian concentration gives
\begin{equation}
    \P\left\{\left|\widetilde f-\E\widetilde f\right|>t\right\}\le2\exp\left(-t^2/\left(2s_2\ell^2\right)\right).
\end{equation}
Since $\P\left(\mathcal A\right)\ge1/2$, conditioning costs at most a factor two. Integrating the tail bounds the change of mean by $C_0\ell$; absorbing this shift into the tail constants gives the same Gaussian bound about the conditional mean. Since $\widetilde f=f\left(\whw\right)$ on $\mathcal A$, substitution of \eqref{eq:solutionlip} proves \eqref{eq:concentration}, with constants independent of $\alpha$, the graph, and $N$.
\end{proof}

\section{Real-World Experiments}
\label{app:real}

This section gives the details of the experimental section Section~\ref{sec:real} of the main paper. The code, including the download of the five datasets, is provided as supplementary material.

\subsection{Protocol}

At each node, the data are split at random into a pool (70\%) and a held-out test set (30\%). Features are centred and standardized with the mean and standard deviation of the union of all pools. We then form the signed observations $\vz=y\vx$ and compute their mean $\vmu^{(i)}$ and covariance $\mSigma^{(i)}$ on the pool of each node. These moments, the sample sizes $n_i$ and the Laplacian $\mL^{(\mathcal G)}$ give the prediction of Theorem~\ref{thm:risk}.

Each training set consists of $n_i$ samples drawn with replacement from the pool of node $i$. The pool is then exactly the population of the theory. For every draw, we compute two errors: the Gaussian formula~\eqref{eq:testerror} with the pool moments, which is the quantity predicted by the theory, and the $0/1$ error on the held-out test set. Each setting uses 100 independent draws, shared across the grid of values of $\alpha$.

\subsection{Datasets, partitions and graphs}

Table~\ref{tab:real_settings} lists the 55 settings. MNIST and Fashion-MNIST use the PCA of the pixels. For rotations, the PCA is fitted on the rotated images of all nodes. CIFAR-10 uses a single layer of random convolutional features \citep{coates2011analysis}: 64 whitened random patches of size $6\times6\times3$, stride 2, a shifted ReLU and sum pooling over the four quadrants, followed by PCA. No pretrained network is used. For Adult, numeric features (log-transformed when skewed) and one-hot categorical features are reduced by PCA. The attribute that defines the silos is removed, and the classes are balanced within each silo.

We build the nodes in four ways. In random splits, the samples of one task are distributed at random, so all nodes have the same distribution. In rotation settings, node $k$ receives images rotated by $30(k-1)$ degrees, $k=1,\ldots,6$. The task is the same at every node, but the features differ. In Dirichlet settings, each class of the binary task is a group of original classes, for example vehicles vs.\ animals in CIFAR-10. Each node draws its proportions of the original classes within each group from a symmetric Dirichlet distribution with parameter $a$. For $a=100$, all nodes have nearly the same proportions. For $a=0.1$, most samples of a node come from one or two original classes. The nodes receive disjoint samples. In alignment settings, nodes 1--3 solve one task and nodes 4--6 another: upright and $90^\circ$-rotated digits for MNIST, and two different pairs of vehicle and animal classes for CIFAR-10.

Besides paths, cycles and complete graphs, we use $k$-nearest-neighbour graphs built from simple statistics of each node. For node $i$, we form the vector containing the mean of $\vz$ and the mean of its squared coordinates over the pool. Two nodes are connected by an edge of unit weight if one of them is among the $k$ nodes closest to the other in Euclidean distance between these vectors.

\begin{table}[h]
\centering\small
\caption{The 55 real-data settings. Unless stated otherwise, $N=6$, $p=100$ and $n_i=250$. Tasks are written as positive vs.\ negative class. Dirichlet partitions use $a\in\{0.1,0.5,100\}$. Complete graphs have total edge weight 6 for random splits, 5 for rotations and $N$ otherwise; all other graphs have unit weights. \#: number of settings.}
\label{tab:real_settings}
\setlength{\tabcolsep}{5pt}
\begin{tabular}{@{}l>{\raggedright\arraybackslash}p{3.6cm}>{\raggedright\arraybackslash}p{5.4cm}>{\raggedright\arraybackslash}p{3.1cm}r@{}}
\toprule
Dataset & Partition & Tasks & Graphs & \# \\
\midrule
MNIST & random split & 3 vs.\ 5 & path, cycle, complete & 3 \\
 & rotations ($0^\circ$--$150^\circ$) & 3/5, 4/9, 7/9, 2/7, 1/7 & path; complete for 3/5 & 6 \\
 & Dirichlet & digits $<5$ vs.\ $\ge5$ & complete, 2-NN & 6 \\
 & alignment ($0^\circ$ and $90^\circ$) & 3 vs.\ 5 & two cycles & 2 \\
\addlinespace
F-MNIST & random split & pullover vs.\ coat & path, cycle, complete & 3 \\
 & rotations ($0^\circ$--$150^\circ$) & pullover/coat, T-shirt/shirt, sneaker/boot, sandal/sneaker, coat/shirt & path; complete for pullover/coat & 6 \\
 & Dirichlet & T-shirt, pullover, coat, shirt vs.\ other items & complete, 2-NN & 6 \\
\addlinespace
CIFAR-10 & random split & automobile vs.\ truck & path, cycle, complete & 3 \\
 & Dirichlet & vehicles vs.\ animals & complete, 2-NN & 6 \\
 & alignment & airplane, ship vs.\ bird, frog; automobile, truck vs.\ cat, dog & two cycles & 2 \\
\addlinespace
HAR & 10 groups of 3 users ($N=10$, $p=30$, $n_i=75$) & sitting vs.\ standing; upstairs vs.\ downstairs & complete, 3-NN & 4 \\
 & 30 single users ($N=30$, $p=12$, $n_i=30$) & same & complete, 3-NN & 4 \\
\addlinespace
Adult & 6 employment sectors ($p=40$, $n_i$ from 140 to 400) & income above \$50{,}000 & complete, 2-NN & 2 \\
 & 5 age groups ($N=5$, $p=40$, $n_i$ from 284 to 400) & same & complete, 2-NN & 2 \\
\midrule
 & & & & 55 \\
\bottomrule
\end{tabular}
\end{table}

\subsection{Choice of $\alpha$ and baselines}

Table~\ref{tab:tuning} uses 20 training draws per setting. All rules choose $\alpha$ on the grid above.
\begin{itemize}
 \item \emph{CV}: 5-fold cross-validation of the node-averaged $0/1$ error on the training data.
 \item \emph{Plug-in theory}: minimizes the prediction of Theorem~\ref{thm:risk}, computed with the training-sample mean and the Ledoit--Wolf covariance estimate of each node.
 \item \emph{Oracle theory}: the same, with the pool moments.
 \item \emph{Consensus}: GTVMin with $\alpha=10^4$.
 \item \emph{Fine-tuning}: $\vw^{(i)}=\argmin_{\vw}L^{(i)}(\vw)+\lambda\left\|\vw-\vw_{\mathrm{con}}\right\|_2^2$, where $\vw_{\mathrm{con}}$ is the consensus model and $\lambda$ is chosen by 5-fold cross-validation.
 \item \emph{Mixture}: GTVMin on the complete graph, with $\alpha$ chosen by 5-fold cross-validation.
 \item \emph{Best (hindsight)}: for each draw, the lowest test error over the grid.
\end{itemize}
Table~\ref{tab:tuning_full} reports the same results with their confidence intervals.

\begin{table}[h]
\centering\small
\caption{Test error (\%) of Table~\ref{tab:tuning} with 95\% confidence intervals over 20 training draws.}
\label{tab:tuning_full}
\begin{tabular}{@{}lcccccc@{}}
\toprule
Method & \makecell{MNIST\\rot.} & \makecell{FMN.\\rot.} & \makecell{CIFAR\\subp.} & \makecell{CIFAR\\Dir.} & \makecell{HAR\\users} & \makecell{Adult\\silos} \\
\midrule
Local & 10.97$\pm$0.23 & 22.32$\pm$0.42 & 17.35$\pm$0.31 & 15.41$\pm$0.36 & 20.77$\pm$0.72 & 25.77$\pm$0.34 \\
Consensus & 18.43$\pm$0.24 & 21.88$\pm$0.24 & 12.19$\pm$0.13 & 11.43$\pm$0.17 & 14.85$\pm$0.22 & 23.15$\pm$0.18 \\
Fine-tuning & 9.44$\pm$0.20 & 17.58$\pm$0.27 & 11.88$\pm$0.20 & 10.65$\pm$0.20 & 14.49$\pm$0.34 & 23.44$\pm$0.26 \\
Mixture & 9.08$\pm$0.18 & 17.35$\pm$0.22 & 11.84$\pm$0.20 & 10.58$\pm$0.21 & 14.74$\pm$0.37 & 23.56$\pm$0.22 \\
\midrule
\multicolumn{7}{@{}l}{\emph{GTVMin, $\alpha$ chosen by}} \\
\quad CV (5-fold) & 8.47$\pm$0.21 & 17.90$\pm$0.29 & 11.64$\pm$0.14 & 10.52$\pm$0.19 & 14.37$\pm$0.27 & 23.51$\pm$0.30 \\
\quad Plug-in theory & 8.37$\pm$0.18 & 17.86$\pm$0.24 & 11.81$\pm$0.19 & 10.44$\pm$0.18 & 14.39$\pm$0.29 & 23.78$\pm$0.28 \\
\quad Oracle theory & 8.37$\pm$0.18 & 17.83$\pm$0.23 & 11.53$\pm$0.17 & 10.37$\pm$0.19 & 14.22$\pm$0.26 & 23.29$\pm$0.20 \\
\midrule
\quad Best (hindsight) & 8.29$\pm$0.17 & 17.73$\pm$0.22 & 11.45$\pm$0.17 & 10.25$\pm$0.18 & 13.91$\pm$0.23 & 22.95$\pm$0.17 \\
\bottomrule
\end{tabular}
\end{table}

\subsection{Where the gap comes from}

We take the i.i.d.\ split of Fashion-MNIST (pullover vs.\ coat) on the path graph. At $\alpha=0$, the Gaussian formula gives an error of $23.2\%$ on the real data, against a prediction of $20.7\%$. We then replace the data of each node by Gaussian samples with the same mean and covariance. The error becomes $20.4\%$, against a prediction of $20.5\%$, and at $\alpha=0.05$ it is $16.50\%$ against $16.54\%$. The gap therefore comes from the shape of the real data and not from the finite dimension.

The quantity that differs is the spread of the whitened squared norm $\|\mSigma^{-1/2}(\vz-\vmu)\|_2^2/p$: its standard deviation is $0.80$ on the real data and $0.14$ on the Gaussian samples. Normalizing each sample to a fixed $\ell_2$ norm does not reduce the gap, because it does not change this spread in whitened coordinates. Replacing the noise term $r_i\mSigma^{(i)}$ in the derivation by its heteroscedastic version $\E[(1-\vz^\top\vm^{(i)})^2(\vz-\vmu^{(i)})(\vz-\vmu^{(i)})^\top]$ reduces the gap only slightly. The gap is largest at $\alpha=0$, where each node has the largest effective aspect ratio, and it is small near consensus.

\subsection{Additional figures}

\begin{figure}[h]\centering
\includegraphics[width=\textwidth]{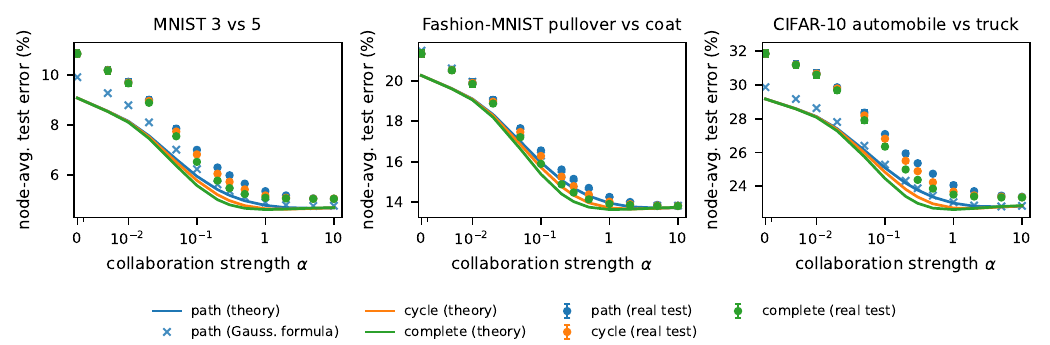}
\caption{Random splits into $N=6$ nodes with $p=100$ and $n_i=250$, on a path, a cycle and a complete graph with total edge weight 6. Lines: Theorem~\ref{thm:risk}. Dots: held-out test error with 95\% intervals over 100 draws. Crosses: Gaussian formula~\eqref{eq:testerror} for the path. As in Corollary~\ref{cor:homogeneous}, the error decreases towards that of the pooled model.}
\label{fig:real_homogeneous}
\end{figure}

\begin{figure}[h]\centering
\includegraphics[width=0.75\textwidth]{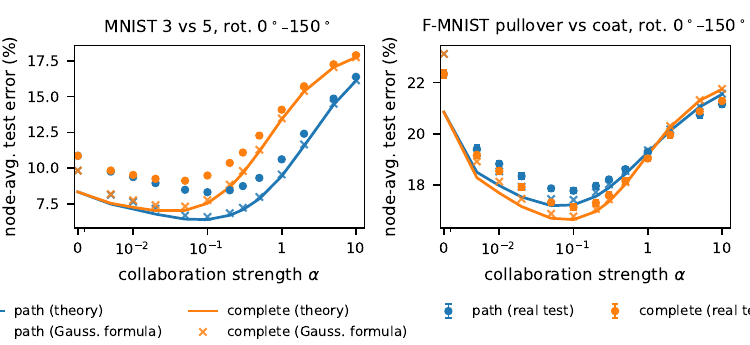}
\caption{Rotation shift: node $k$ sees the images rotated by $30(k-1)$ degrees, $N=6$, $p=100$, $n_i=250$, on a unit-weight path and a complete graph with total weight 5. Lines, dots and crosses as in Figure~\ref{fig:real_homogeneous}. The best $\alpha$ is intermediate, and on MNIST full collaboration is worse than local training.}
\label{fig:real_rotation}
\end{figure}

\begin{figure}[h]\centering
\includegraphics[width=0.75\textwidth]{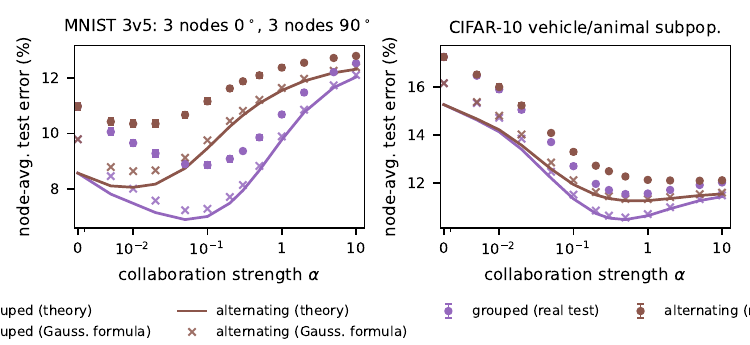}
\caption{Graph--task alignment on real data, with the two cycles of Figure~\ref{fig:alignment}. Left: MNIST 3 vs.\ 5, with nodes 1--3 upright and nodes 4--6 rotated by $90^\circ$. Right: CIFAR-10, nodes 1--3 separate airplanes and ships from birds and frogs, nodes 4--6 automobiles and trucks from cats and dogs. Both cycles have the same Laplacian spectrum and use the same training draws. The grouped cycle, with two edges between different tasks instead of six, has the lower error, as predicted.}
\label{fig:real_alignment}
\end{figure}

\begin{figure}[h]\centering
\includegraphics[width=\textwidth]{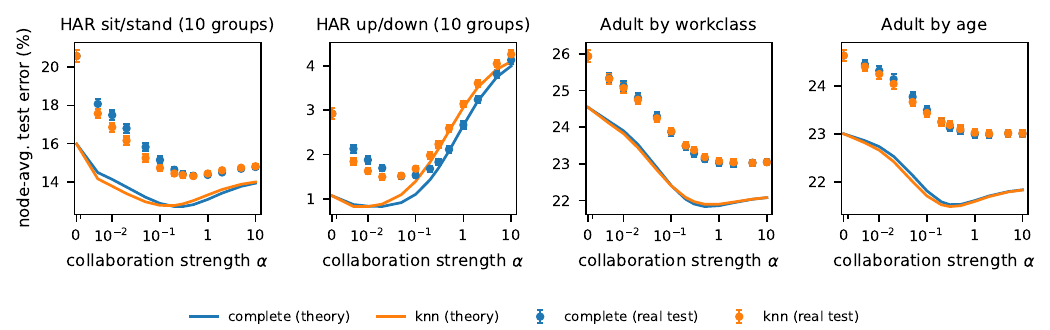}
\caption{Natural partitions. HAR: 10 nodes of 3 users each, $p=30$, $n_i=75$. Adult: 6 employment sectors or 5 age groups, $p=40$, unequal $n_i$. Complete graphs with total weight $N$, and $k$-NN graphs built from pool statistics ($k=3$ for HAR, $k=2$ for Adult). Lines and dots as in Figure~\ref{fig:real_homogeneous}.}
\label{fig:real_natural}
\end{figure}

\begin{figure}[h]\centering
\includegraphics[width=0.75\textwidth]{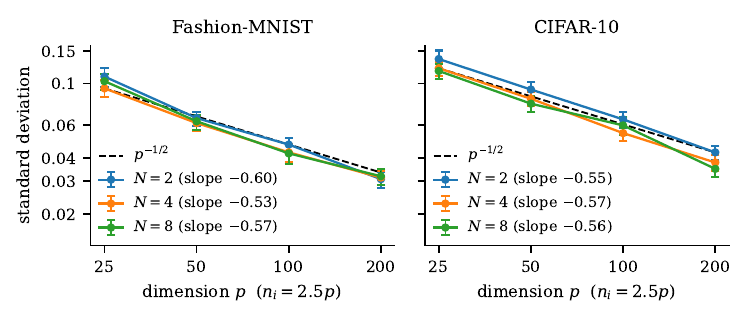}
\caption{Theorem~\ref{thm:concentration} on real data: standard deviation over 200 training draws of the unit-norm observable $N^{-1/2}\sum_i\ve_1^\top\whw^{(i)}$, for random splits of Fashion-MNIST (pullover vs.\ coat) and CIFAR-10 (automobile vs.\ truck). Unweighted cycle, $\alpha=0.3$, $n_i=2.5p$, bootstrap 95\% intervals. The fitted slopes, given in the legend, are close to $-1/2$ for every $N$.}
\label{fig:real_concentration}
\end{figure}

\subsection{All settings}

Table~\ref{tab:real_all} gives the results of every setting. The columns are:
\begin{itemize}
 \item $|\Delta|_{\mathrm G}$ and $|\Delta|_{\mathrm T}$: the largest difference over the grid between the prediction and, respectively, the Gaussian formula~\eqref{eq:testerror} and the held-out test error, in pp;
 \item $\alpha^{\mathrm{de}}$ and $\alpha^\star_{\mathrm{test}}$: the value of $\alpha$ selected by the theory, and the one with the lowest mean test error;
 \item regret: the extra test error at $\alpha^{\mathrm{de}}$, in pp;
 \item local, cons.\ and best: the test errors (\%) at $\alpha=0$, at $\alpha=10$, and at $\alpha^\star_{\mathrm{test}}$;
 \item NC: the non-Gaussianity index, which equals 1 for Gaussian data.
\end{itemize}

{\footnotesize\setlength{\tabcolsep}{3pt}
\input{latex/tab_full.tex}}
\captionof{table}{Results for all 55 real-data settings.}\label{tab:real_all}

\end{document}

%% file: math_commands.tex
\def\R{\mathbb{R}}
\def\E{\mathbb{E}}
\renewcommand{\P}{\mathbb{P}}

\renewcommand{\epsilon}{\varepsilon}
\def\va{\mathbf{a}}
\def\vb{\mathbf{b}}
\def\vc{\mathbf{c}}
\def\vd{\mathbf{d}}
\def\ve{\mathbf{e}}

\def\vm{\mathbf{m}}

\def\vq{\mathbf{q}}
\def\vr{\mathbf{r}}

\def\vu{\mathbf{u}}
\def\vv{\mathbf{v}}
\def\vw{\mathbf{w}}
\def\vx{\mathbf{x}}
\def\vy{\mathbf{y}}
\def\vz{\mathbf{z}}
\def\mA{\mathbf{A}}
\def\mB{\mathbf{B}}
\def\mC{\mathbf{C}}
\def\mD{\mathbf{D}}
\def\mE{\mathbf{E}}
\def\mF{\mathbf{F}}
\def\mG{\mathbf{G}}
\def\mH{\mathbf{H}}
\def\mI{\mathbf{I}}
\def\mJ{\mathbf{J}}
\def\mK{\mathbf{K}}
\def\mL{\mathbf{L}}
\def\mM{\mathbf{M}}

\def\mP{\mathbf{P}}
\def\mQ{\mathbf{Q}}
\def\mR{\mathbf{R}}
\def\mS{\mathbf{S}}
\def\mT{\mathbf{T}}
\def\mU{\mathbf{U}}
\def\mV{\mathbf{V}}
\def\mW{\mathbf{W}}
\def\mX{\mathbf{X}}

\def\mZ{\mathbf{Z}}
\def\vmu{\bm{\mu}}
\def\veta{\bm{\eta}}

\def\vdelta{\bm{\delta}}

\def\vomega{\bm{\omega}}

\def\mSigma{\mathbf{\Sigma}}
\def\mGamma{\mathbf{\Gamma}}

\DeclareMathOperator{\tr}{tr}

\DeclareMathOperator{\sign}{sign}
\DeclareMathOperator{\stack}{stack}
\DeclareMathOperator{\blkdiag}{blockdiag}
\DeclareMathOperator{\diag}{diag}

\DeclareMathOperator*{\argmin}{arg\,min}

\newcommand{\ie}{\emph{i.e.},\ }
\newcommand{\vone}{\bm{1}}
\newcommand{\vzero}{\bm{0}}
\newcommand{\whw}{\widehat{\vw}}

\newcommand{\calE}{\mathcal{E}}

%% file: results/summary.tex
\newcommand{\NumGaussianRepeats}{240}

%% file: latex/tab_full.tex
\begin{longtable}{lrrrrrrrrrrr}
\toprule
Setting [graph] & $N$ & $p$ & $|\Delta|_{\text{G}}$ & $|\Delta|_{\text{T}}$ & $\alpha^{\text{de}}$ & $\alpha^{\star}_{\text{test}}$ & regret & local & cons. & best & NC \\
\midrule
\endhead
iid; CIFAR auto/truck [complete] & 6 & 100 & 0.70 & 2.69 & 1 & 5 & 0.17 & 31.9 & 23.3 & 23.3 & 3.1 \\
iid; CIFAR auto/truck [cycle] & 6 & 100 & 0.70 & 2.69 & 2 & 10 & 0.11 & 31.9 & 23.3 & 23.3 & 3.1 \\
iid; CIFAR auto/truck [path] & 6 & 100 & 0.70 & 2.69 & 5 & 10 & 0.07 & 31.9 & 23.4 & 23.4 & 3.1 \\
iid; F-MNIST pullover/coat [complete] & 6 & 100 & 1.23 & 1.07 & 1 & 10 & 0.10 & 21.3 & 13.8 & 13.8 & 3.7 \\
iid; F-MNIST pullover/coat [cycle] & 6 & 100 & 1.23 & 1.07 & 2 & 10 & 0.08 & 21.3 & 13.8 & 13.8 & 3.7 \\
iid; F-MNIST pullover/coat [path] & 6 & 100 & 1.23 & 1.07 & 5 & 10 & 0.03 & 21.3 & 13.8 & 13.8 & 3.7 \\
iid; MNIST 3/5 [complete] & 6 & 100 & 0.84 & 1.77 & 1 & 5 & 0.04 & 10.9 & 5.1 & 5.1 & 2.0 \\
iid; MNIST 3/5 [cycle] & 6 & 100 & 0.84 & 1.77 & 2 & 5 & 0.03 & 10.9 & 5.1 & 5.1 & 2.0 \\
iid; MNIST 3/5 [path] & 6 & 100 & 0.84 & 1.77 & 5 & 10 & 0.02 & 10.9 & 5.1 & 5.1 & 2.0 \\
Dirichlet; CIFAR vehicle/animal, $\beta$=100.0 [complete] & 6 & 100 & 0.87 & 1.78 & 1 & 5 & 0.11 & 20.7 & 12.5 & 12.5 & 3.1 \\
Dirichlet; CIFAR vehicle/animal, $\beta$=100.0 [knn] & 6 & 100 & 0.87 & 1.78 & 1 & 5 & 0.18 & 20.7 & 12.5 & 12.5 & 3.1 \\
Dirichlet; F-MNIST upper/rest, $\beta$=100.0 [complete] & 6 & 100 & 1.19 & 3.70 & 1 & 10 & 0.08 & 10.7 & 6.3 & 6.3 & 3.3 \\
Dirichlet; F-MNIST upper/rest, $\beta$=100.0 [knn] & 6 & 100 & 1.19 & 3.70 & 1 & 10 & 0.09 & 10.7 & 6.3 & 6.3 & 3.3 \\
Dirichlet; MNIST $<5$/$\ge5$, $\beta$=100.0 [complete] & 6 & 100 & 0.96 & 2.02 & 1 & 5 & 0.13 & 26.0 & 17.2 & 17.2 & 2.2 \\
Dirichlet; MNIST $<5$/$\ge5$, $\beta$=100.0 [knn] & 6 & 100 & 0.96 & 2.02 & 1 & 10 & 0.17 & 26.0 & 17.2 & 17.2 & 2.2 \\
Dirichlet; CIFAR vehicle/animal, $\beta$=0.1 [complete] & 6 & 100 & 0.92 & 2.12 & 0.2 & 0.3 & 0.08 & 15.5 & 11.4 & 10.5 & 3.1 \\
Dirichlet; CIFAR vehicle/animal, $\beta$=0.1 [knn] & 6 & 100 & 0.92 & 2.12 & 0.2 & 0.3 & 0.05 & 15.5 & 11.4 & 10.4 & 3.1 \\
Dirichlet; CIFAR vehicle/animal, $\beta$=0.5 [complete] & 6 & 100 & 1.03 & 2.41 & 0.3 & 0.5 & 0.10 & 16.5 & 11.3 & 11.0 & 3.1 \\
Dirichlet; CIFAR vehicle/animal, $\beta$=0.5 [knn] & 6 & 100 & 1.03 & 2.41 & 0.5 & 0.5 & 0.00 & 16.5 & 11.3 & 10.9 & 3.1 \\
Dirichlet; F-MNIST upper/rest, $\beta$=0.1 [complete] & 6 & 100 & 0.79 & 2.29 & 0.1 & 0.1 & 0.00 & 4.6 & 4.1 & 3.1 & 3.9 \\
Dirichlet; F-MNIST upper/rest, $\beta$=0.1 [knn] & 6 & 100 & 0.79 & 2.29 & 0.1 & 0.1 & 0.00 & 4.6 & 4.1 & 3.1 & 3.9 \\
Dirichlet; F-MNIST upper/rest, $\beta$=0.5 [complete] & 6 & 100 & 1.19 & 3.06 & 0.3 & 0.3 & 0.00 & 8.2 & 5.9 & 5.5 & 3.8 \\
Dirichlet; F-MNIST upper/rest, $\beta$=0.5 [knn] & 6 & 100 & 1.19 & 3.06 & 0.3 & 0.5 & 0.01 & 8.2 & 5.8 & 5.6 & 3.8 \\
Dirichlet; MNIST $<5$/$\ge5$, $\beta$=0.1 [complete] & 6 & 100 & 1.20 & 3.07 & 0.05 & 0.1 & 0.13 & 11.5 & 15.0 & 9.2 & 2.5 \\
Dirichlet; MNIST $<5$/$\ge5$, $\beta$=0.1 [knn] & 6 & 100 & 1.20 & 3.07 & 0.05 & 0.1 & 0.12 & 11.5 & 14.7 & 9.3 & 2.5 \\
Dirichlet; MNIST $<5$/$\ge5$, $\beta$=0.5 [complete] & 6 & 100 & 1.59 & 2.68 & 0.1 & 0.3 & 0.25 & 18.4 & 16.6 & 15.0 & 3.0 \\
Dirichlet; MNIST $<5$/$\ge5$, $\beta$=0.5 [knn] & 6 & 100 & 1.59 & 2.68 & 0.1 & 0.2 & 0.11 & 18.4 & 16.6 & 14.8 & 3.0 \\
rotation; F-MNIST pullover/coat [complete] & 6 & 100 & 2.28 & 1.49 & 0.1 & 0.1 & 0.00 & 22.3 & 21.3 & 17.1 & 5.6 \\
rotation; F-MNIST pullover/coat [path] & 6 & 100 & 2.28 & 1.49 & 0.05 & 0.1 & 0.09 & 22.3 & 21.2 & 17.8 & 5.6 \\
rotation; MNIST 3/5 [complete] & 6 & 100 & 1.48 & 2.52 & 0.02 & 0.05 & 0.13 & 10.9 & 17.9 & 9.1 & 3.3 \\
rotation; MNIST 3/5 [path] & 6 & 100 & 1.48 & 2.52 & 0.1 & 0.1 & 0.00 & 10.9 & 16.4 & 8.3 & 3.3 \\
rotation; F-MNIST 0v6 [path] & 6 & 100 & 2.81 & 0.97 & 0.2 & 0.3 & 0.05 & 23.0 & 19.6 & 18.7 & 5.8 \\
rotation; F-MNIST 4v6 [path] & 6 & 100 & 2.33 & 1.21 & 0.1 & 0.2 & 0.06 & 23.0 & 20.2 & 17.9 & 5.7 \\
rotation; F-MNIST 5v7 [path] & 6 & 100 & 4.38 & 2.40 & 0.05 & 0.05 & 0.00 & 11.3 & 14.8 & 8.1 & 6.3 \\
rotation; F-MNIST 7v9 [path] & 6 & 100 & 1.95 & 2.64 & 0.05 & 0.1 & 0.02 & 8.2 & 7.3 & 6.1 & 5.1 \\
rotation; MNIST 1v7 [path] & 6 & 100 & 1.93 & 3.16 & 0.1 & 0.1 & 0.00 & 3.5 & 2.6 & 1.9 & 6.7 \\
rotation; MNIST 2v7 [path] & 6 & 100 & 0.98 & 2.73 & 0.05 & 0.1 & 0.13 & 5.5 & 6.8 & 3.6 & 3.3 \\
rotation; MNIST 4v9 [path] & 6 & 100 & 1.66 & 2.26 & 0.05 & 0.2 & 0.24 & 9.6 & 14.5 & 6.8 & 4.1 \\
rotation; MNIST 7v9 [path] & 6 & 100 & 3.02 & 2.14 & 0.05 & 0.1 & 0.09 & 9.6 & 12.5 & 7.7 & 5.1 \\
alignment; CIFAR subpop. [alternating] & 6 & 100 & 0.89 & 1.99 & 0.5 & 5 & 0.18 & 17.3 & 12.1 & 12.1 & 3.3 \\
alignment; CIFAR subpop. [grouped] & 6 & 100 & 0.89 & 1.99 & 0.5 & 0.5 & 0.00 & 17.3 & 12.0 & 11.5 & 3.3 \\
alignment; MNIST 0/90$^\circ$ [alternating] & 6 & 100 & 1.22 & 2.41 & 0.01 & 0.01 & 0.00 & 11.0 & 12.8 & 10.3 & 2.9 \\
alignment; MNIST 0/90$^\circ$ [grouped] & 6 & 100 & 1.22 & 2.41 & 0.05 & 0.1 & 0.03 & 11.0 & 12.5 & 8.9 & 2.9 \\
HAR; sit/stand, 10 groups [complete] & 10 & 30 & 3.15 & 4.61 & 0.3 & 0.5 & 0.08 & 20.6 & 14.8 & 14.3 & 2.8 \\
HAR; sit/stand, 10 groups [knn] & 10 & 30 & 3.15 & 4.61 & 0.2 & 0.5 & 0.14 & 20.6 & 14.8 & 14.3 & 2.8 \\
HAR; sit/stand, 30 users [complete] & 30 & 12 & 2.68 & 4.59 & 0.02 & 0.5 & 1.29 & 23.7 & 20.0 & 19.0 & 1.8 \\
HAR; sit/stand, 30 users [knn] & 30 & 12 & 2.68 & 4.59 & 0.01 & 0.2 & 1.28 & 23.7 & 20.0 & 19.0 & 1.8 \\
HAR; up/down, 10 groups [complete] & 10 & 30 & 0.61 & 1.86 & 0.02 & 0.05 & 0.18 & 2.9 & 4.1 & 1.5 & 1.6 \\
HAR; up/down, 10 groups [knn] & 10 & 30 & 0.61 & 1.86 & 0.01 & 0.02 & 0.13 & 2.9 & 4.3 & 1.5 & 1.6 \\
HAR; up/down, 30 users [complete] & 30 & 12 & 0.49 & 1.50 & 0.01 & 0.05 & 0.15 & 2.2 & 4.3 & 0.9 & 1.1 \\
HAR; up/down, 30 users [knn] & 30 & 12 & 0.49 & 1.50 & 0.005 & 0.02 & 0.11 & 2.2 & 4.3 & 0.9 & 1.1 \\
Adult; age [complete] & 5 & 40 & 1.63 & 1.68 & 0.3 & 2 & 0.14 & 24.6 & 23.0 & 23.0 & 7.0 \\
Adult; age [knn] & 5 & 40 & 1.63 & 1.73 & 0.3 & 2 & 0.19 & 24.6 & 23.0 & 23.0 & 7.0 \\
Adult; workclass [complete] & 6 & 40 & 6.77 & 1.47 & 0.5 & 2 & 0.12 & 25.9 & 23.1 & 23.0 & 7.4 \\
Adult; workclass [knn] & 6 & 40 & 6.77 & 1.50 & 0.5 & 5 & 0.15 & 25.9 & 23.0 & 23.0 & 7.4 \\
\bottomrule
\end{longtable}